%% file: arxiv.tex
\documentclass{article}
\usepackage{times}

\usepackage{soul}
\usepackage{url}
\usepackage[hidelinks]{hyperref}
\usepackage[utf8]{inputenc}
\usepackage[small]{caption}
\usepackage{graphicx}
\usepackage{amsmath}
\usepackage{booktabs}
\usepackage{amsthm}

\usepackage{cite}
\usepackage{natbib}
\usepackage{misc0}
 \usepackage{algorithm}
\usepackage{algorithmic}
 
\usepackage{tikz}
 \definecolor{Black}  {RGB}{0,0,0}
\usepackage{thm-restate} 
\usepackage{url}
\usepackage[utf8]{inputenc} % usually not needed (loaded by default)
\usepackage[T1]{fontenc}
\theoremstyle{definition}
\newtheorem{theorem}{Theorem}
\newtheorem{example}[theorem]{Example}
\newtheorem{definition}[theorem]{Definition}

\newtheorem{claim}[theorem]{Claim}

\newcommand{\safe}{safe\xspace} 
\newcommand{\atom}{\ensuremath{p}\xspace} 
\newcommand{\dependencygraph}{dependency graph\xspace}

\newcommand{\BK}{\Bmc}
\newcommand{\ant}[1]{{\sf ant}(#1)}
\newcommand{\cons}[1]{{\sf con}(#1)}
\newcommand{\formula}[1]{\ensuremath{\Fmc_{#1}}\xspace}

\newcommand{\literal}{\ensuremath{l}\xspace}

\usepackage{marvosym}
\usepackage{authblk}

\title{Finding Common Ground   for Incoherent Horn Expressions}
\author[1]{Ana Ozaki}
\author[1]{Anum Rehman} 
\author[1]{Philip Turk}
\author[1]{Marija Slavkovik} 
 \affil[1]{University of Bergen \\
\{ana.ozaki, marija.slavkovik\}@uib.no }

\begin{document}

\maketitle

 \begin{abstract}
% \textcolor{red}{**I think we need a more technical abstract, risk of random people reviewing here**}
 Autonomous systems that operate in a shared environment with people 
 need to be able to follow the rules  of the  society they occupy.  
 While laws are unique for one society, different  people and institutions may use different  
 rules to guide their conduct.  We study the problem of 
 reaching a common ground among possibly incoherent rules of conduct.
%   an autonomous system should obey. 
% Our  
% algorithm  combines   incoherent behaviour rules represented in propositional Horn.  
 We formally define a notion of common ground and discuss the main properties of this notion.
Then, we identify three sufficient conditions on the class of Horn expressions for which common grounds are
guaranteed to exist. We
  provide a polynomial time 
algorithm that computes  common grounds, under these conditions. % for Horn expressions in this class.
We also show that if any of the three conditions is removed then 
common grounds for the resulting (larger) class may not exist.
%there is no algorithm that is guaranteed to 
%compute common grounds. % for the resulting (larger) class.
% of 
%Horn expressions. 
 %for 
%there is 
%removing the conditions for this class of Horn expressions

\end{abstract}

\section{Introduction} %\nb{maybe we should think about a catchy title} 
Systems   capable of some level of autonomous operation  should be built 
 to respect the moral norms and values of the society in which they operate~\citep{Dignum2019ch,WinfieldMPE2019,wallachBook,Moor:2006,bremner2018}. % or that at least considerations 
%should be taken when decision-making is automated\footnote{See \url{https://futureoflife.org/national-international-ai-strategies/?cn-reloaded=1} 
%for a comprehensive list of national strategies on AI.}. 
%Whose %removed what, not sure if it should be which or what
%values should an autonomous system 
%follow? 
If there are multiple stakeholders supplying different rules, how should possible inconsistencies among them be resolved? 

The question of whether or not moral conflicts do really exist  has been long argued in moral philosophy \citep{Donagan84}. 
It has also been argued that people do not tend to follow  moral theories, but rules of thumb when choosing what to 
do in a morally sensitive situation \citep{Horty94}. Numerous conflicts do arise among rules. A normative conflict is a situation in which an agent ought to perform two actions that cannot both be performed \citep{Horty2003}.

%An autonomous system that serves many masters, must be equipped with a method for resolving possible incoherent and conflicting requirements in their input. What method should be used?

There is an entire field that studies the resolution of normative conflicts \citep{Santos2018}.   \cite{Baum20} and \cite{Rahwan2018} argue that some form of a social 
choice approach is needed to decide and \cite{Botan2021} consider judgment aggregation as a method.  \cite{NoothigattuGADR18} consider learning 
the moral preferences of people and using them to vote on what is the right 
thing for an autonomous system to do. \cite{Liao2018} and \cite{LiaoST19} propose 
an argumentation based approach.  

In this paper we do not aim to solve the problem of moral normative conflicts, 
moral social choice or deontic conflict resolution. We are interested in 
exploring a very practical, admittedly limited,  way of reaching a common ground on 
sets of rules.
We assume that each 
stakeholder with an interest in governing the behavior of an autonomous system contributes  
a set of rules which the system should implement. However, these rules are considered to be under-specified. 
Namely, we assume that  the stakeholder does not exclude that exceptions to the rules exist. We design an 
algorithm that ``corrects'' the rules supplied by one stakeholder with exceptions raised by another stakeholder.  

We only consider rules expressed in Horn logic. This is motivated primarily by the availability of tools that handle this logic, 
when compared with deontic logics and formalisms based on nonmonotonic reasoning,  which are  two other natural choices to represent behavior rules.  We study a notion of common ground and under which conditions it  exists. 

To illustrate the main ideas and approaches we follow in building our  algorithm, consider as an example a police robot\footnote{Similar examples can be found in \citep{BjorgenMBHHLLDS18}} that   
 detects a potentially illegal activity in a supermarket. 
\begin{example}\label{ex:consensus}\upshape
Assume that the police robot has detected %marijuana 
 smoke  and a child %\nb{what about using minor instead?}
%teen by child? 
who is smoking. %small change ok? 
The robot has made the following deduction: %\footnote{Formulated in propositional logic, for simplicity. }:
$$\text{A child is smoking in a forbidden to smoke area.}$$
%In symbols, ${\sf Child(Mike)}, {\sf illegalDrugUse(Mike,drugX)}$.
Now assume we have  stakeholders, which we may also refer to as ``agents'', % throughout the paper, 
with the following recommendations. 
The first agent puts forward the rule:  
 ``if there is an illegal activity, the police should be informed''.  
 In symbols: \begin{equation}\label{callPo}{\sf illegalActivity}\rightarrow {\sf policeCall}\end{equation}
%$$\forall xy ({\sf illegalActivity}(x,y)\rightarrow {\sf policeCall})$$
The second agent points out that it is a child who is smoking, 
and so,  the parents are the ones who should decide if the police should be informed. 
Police officers do exercise an independent judgement as part of best serving the public 
and we would not want to have a more oppressive society with robots that completely eliminate this practice. 
More specifically, we have the rule 
 ``if there is an illegal activity done by a child then their parents should be called''. In symbols: 
\begin{equation}\label{callPa}({\sf illegalActivity}\wedge {\sf child})\rightarrow {\sf parentsAlert}\end{equation}
%Both of 
The agents agree that smoking in %a supermarket 
a forbidden to smoke area (e.g., bus stop)
is an illegal activity and that, if this happens,
%in such situation 
either the police or the parents 
should be called.
They agree that not both should be called, that is, \[({\sf parentsAlert}\wedge {\sf policeCall})\rightarrow \bot.\] 
Calling both parents and the police is somewhat pointless since the police is obliged to call the parents of the minor. \hfill {\mbox{$\triangleleft$}}
\end{example}

In this example, there is a clear incoherence but not necessarily a conflict  between the 
two agents. %: one thinks that the police should be called and the other thinks that the parents should be alerted.
We consider this an incoherence because the second agent reasons using a more specific rule than the first. %\blu{ % I like this connection ! 
Rules could be even further specialized and take into account 
the case in which the child is not under parental supervision  (e.g., alone in the supermarket). %, for example when it is home alone. 
 
 Consider the rule:
 ``if there is an illegal activity done by a child who is unsupervised (by an adult) then the police should be called''.
 In symbols: 
 \begin{equation}\label{callUn}({\sf illegalActivity}\wedge {\sf child}\wedge {\sf unsupervised})\rightarrow {\sf policeCall}\end{equation}

 A common ground between (\ref{callPo}) and  (\ref{callPa}) can be reached by transforming (\ref{callPo}) into 
  \begin{equation}\label{callPol}
  ({\sf illegalActivity} \wedge {\sf adult}) \rightarrow {\sf policeCall}
  \end{equation}
 
 What is interesting about   
 (\ref{callPol})  is that it is  
 coherent with  both (\ref{callPa}) and (\ref{callUn}). %\hfill {\mbox{$\triangleleft$}}
When there are two applicable rules %that apply %avoiding the term conflict, used in this paper with another meaning
but one is more specialized,  legal theory prescribes the use of the  {\em lex specialis derogat legi generali}  principle meaning ``special law repeals general laws''.  
We explore the possibility of applying this principle as a basis for reaching common grounds.  

What makes one set of rules a common ground? We consider the scenario in which all stakeholders are equally important. A set of rules, represented as a Horn expression, can only be a common ground for their behaviour rules, given as input, 
if it is coherent with them. We formally define  the notion of a common ground for a given set of Horn expressions and discuss the main properties of it.
Common grounds do not always exist.
We identify a large class of Horn expressions for which common grounds are
guaranteed to exist and
  provide a polynomial time 
%We build an 
algorithm that computes  common grounds for Horn expressions in this class.
%under certain restrictions in 
%the input rules. removing the restrictions we impose
%
%identifies pairs of rules that can be made coherent by making them as specific as possible using 
%the available information.  
We assume 
%does not repeat an example that has already been 
%given and that we 
 a common background knowledge for  the  stakeholders that expresses mutually exclusive conditions, for example, being an adult or a child. % which is a set of formulas $X \wedge \bar{X} \models \bot$ declaring $\bar{X}$ a complement of $X$ and vice versa.   
%\end{itemize}
Our main contributions are  
\begin{itemize}
\item the formalisation and discussion of a notion of common ground for incoherent rules (Section~\ref{sec:definitions});
\item the proposal of a polynomial time algorithm for finding common grounds based on the \emph{lex specialis derogat legi generali} legal
principle (Section~\ref{sec:building});
and 
\item an analysis 
of when  a common ground for incoherent rules is guaranteed to exist (Section~\ref{sec:building}).
\end{itemize}
%The main contribution of the paper is..

%\todo[inline]{TODO:contributions.}
%\todo[inline]{Paper structure}

%~ The paper is structured as follows. In Section~\ref{sec:definitions} we define our framework. More specifically, 
%~ we formally express the difference between conflict and (in)coherence. 
%~ We also define a set of postulates that a  Horn expression should satisfy to be considered a 
%~ common ground. % or a stronger notion which we call a full compromise. 
%~ In Section~\ref{sec:building} we present our algorithm and consider the conditions under which 
%~ the algorithm is guaranteed to output  a common ground.  
%In Section~\ref{sec:conflict} we offer a discussion on what having conflicts, rather than incoherence, 
%among rules means and how to handle this situation. 
In Section~\ref{sec:relwork} we discuss related work. In Section~\ref{sec:conclusion} we outline our conclusions and discuss directions for future work.

\section{Rules, Incoherences, and Common Ground}\label{sec:definitions} 
We represent  rule recommendations   using propositional Horn logic. In the following, we briefly introduce the syntax and semantics of %function-free 
%first-order 
propositional Horn expressions and provide basic notions used in this paper. 
Then, we formally introduce our notions of \emph{coherence}, \emph{conflict} and \emph{common ground}. 

%A term $t$ is either a variable, a constant, or an expression of the form $f(t_1,\ldots,t_a)$, where 
%$t_1,\ldots,t_a$ are terms. For example, $f(a,x,g(y,b))$ and $\textit{recommend}(a, \textit{choice}(a,b))$ are terms. 

\subsection{Rules}

\paragraph{Syntax}
An \emph{atom} is a boolean variable. %an expression of the form 
%$P(\vec{t})$ with $P$ a predicate and $\vec{t}$ a list of terms $t_1,\ldots,t_a$,   
%where $a$ is the arity of $P$. An atom is \emph{ground} if all terms 
%occurring in it are constants. 
A \emph{literal} is an atom $p$ or its negation $\neg p$. %We use the notion of a set of formulas and the conjunction of its elements interchangeably. 
A \emph{Horn clause} is a 
%universally quantified 
disjunction of literals %. It is called \emph{Horn} if it 
where  at most one is positive. % literal. 
%A Horn clause 
It is \emph{definite} if it has exactly one positive literal. 
For a given definite Horn clause $\phi$, we define $\ant{\phi} $ to be the set of 
all atoms   such that their negation occurs in
%all  negative literals in 
$\phi$, while $\cons{\phi}$ is the positive literal in $\phi$. 
A definite Horn clause is \emph{non-trivial} if $\cons{\phi}\not\in\ant{\phi}$. 
%\nb{A: changed here because the def says that a Horn clause has at most one positive literals. 
%So we may not have any and thus cannot say "the positive" (it may not be there...)} \
We may treat a set of formulas and the conjunction of its elements interchangeably.
Also, we often write Horn clauses as \emph{rules} of the form 
$$(p_1\wedge\ldots\wedge p_n)\rightarrow q\text{  
or }(p_1\wedge\ldots\wedge p_n)\rightarrow \bot$$ 
the latter are for non-definite clauses. 
%When $\ant{\phi}$ is used as a clause (e.g. in Definition~\ref{def:conflict}), it stands for all clauses $\top\rightarrow\beta$ with $\beta\in\ant{\phi}$.
%
A \emph{Horn expression} is a (finite) set of 
%(first-order) 
Horn clauses. It is called  \emph{definite} if all clauses in it are definite.  
%
%
%\nb{may be moved to preliminaries}
%\begin{definition}[Acyclic Horn]
%\end{definition} 

\paragraph{Semantics}
The semantics is given by interpretations. We denote by 
$\true(\Imc)$ the set of variables assigned to true in an interpretation \Imc. 
We say that \Imc \emph{satisfies} 
a  Horn clause $\phi$
if 
\begin{itemize}
\item $\ant{\phi}\not\subseteq{\sf true}(\Imc)$ or, 
\item in the case $\phi$ is definite, 
$\cons{\phi}\in\true(\Imc)$.
\end{itemize}
It satisfies 
a Horn expression $\formula{}$ if it satisfies all clauses in $\formula{}$. 
A Horn expression $\formula{}$ \emph{entails} a clause $\phi$, written
$\formula{}\models\phi$, if every interpretation that satisfies $\formula{}$ also 
satisfies $\phi$.
%If $\{\psi\}\models\phi$ but $\{\phi\}\not\models\psi$ then we say that
%$\psi$ is \emph{more general} than $\phi$ (or, equivalently, $\phi$ is \emph{more specific} than $\psi$). 
%\nb{... add more for the semantics}
%
%We say that 
%~ A Horn expression  \formula{} is  {\em non-redundant} 
%~ if for all $\phi\in\formula{}$ it is not the case that  $\formula{}\setminus\{\phi\}\models \phi$. 
%~ %We say that 
%~ %A Horn expression $\formula{}$ 
%~ It is \emph{acyclic}
%~ if there is no sequence of clauses $\phi_1, \ldots, \phi_n \in \formula{}$ 
%~ such that $\cons{\phi_i}\in \ant{\phi_{i+1}}$, for all $1\leq i < n$,
%~ and $\phi_1=\phi_n$. 

%We assume that there exist a set of stakeholders, or agents, which we simply enumerate as $i\in\{1,\ldots,n\}$. 
\paragraph{Scenario}
We consider a scenario with multiple stakeholders, or agents. %, which we  enumerate as $i\in\{1,\ldots,n\}$. 
Each  stakeholder $i$
is associated with its own set of   {\em behaviour  rules}, e.g., %such as 
``if there is an illegal activity then call the police''. The stakeholders also share {\em background knowledge}.
Their background knowledge %of the stakeholders
  contains basic  constraints about the world, such as ``a person cannot be an adult and a child at the same time''. 
  Since stakeholders may diverge in their behaviour, 
the goal is to find a Horn expression $\formula{}$ that is a {\em representative} 
of such behaviours that also respects the constraints about the world; %the  Horn expressions, 
we  call $\formula{}$ {\em a common ground}.
 %\nb{changed paragraph above and below}

\medskip

We represent  behaviour rules 
%a
 %or recommendations,  
%represented 
with   definite Horn expressions, denoted
%.her own Horn expression 
$\formula{i}$ for each  stakeholder $i$, and the background knowledge with a set of non-definite Horn clauses.
This way of representing behaviour rules and background knowledge   simplifies the presentation of the technical results while 
capturing a large class of scenarios. 
Also, as discussed earlier, 
%even though simple, 
Horn logic is a convenient formalism because there are several tools for performing automated reasoning.
 %The stakeholders share {\em background knowledge}. Their background knowledge %of the stakeholders
 % contains basic constraints about the world, such as ``a person cannot be an adult and a child at the same time''. 
  %It is complete in the sense that it 
  %relates to all the necessary constraints related to the propositions used for representing %(predicates, functions, terms and constants) 
 %the %behavior %moral 
  %recommendations. 
  %Formally, 
  
  The background knowledge $\BK$ for $\formula{1},\ldots,\formula{n}$
  %of the stakeholders
  is defined as a set of non-definite Horn clauses, %$(\atom_1 \wedge \ldots\wedge\atom_n) \rightarrow \bot$,
  %where $\atom_i$   are 
  built from atoms occurring in $\formula{1},\ldots,\formula{n}$,
  expressing pairwise disjointness constraints
  (e.g., a person cannot be a child and an adult, or a child and a teenager). %\nb{propositional, first-order}
 %We introduce a special notation to relate ``excludent'' atoms in $\BK$ w.r.t. a definite Horn expression $\formula{}$. %\nb{M: this paragraph we may need to rephrase at some point.} 
 For a given atom $\atom$ ocurring in $\BK$, % and a definite Horn expression $\formula{}$,  
 we  define 
 \[\overline{\atom}_{\BK} = \{ q \mid   \BK\models (\atom\wedge q) \rightarrow \bot) \}.\]
 We may omit writing `for $\formula{1},\ldots,\formula{n}$' and the subscript $\cdot_{\BK}$ when this is clear from the context.  
Elements of $\overline{\atom}$ are called \emph{excludents} of $\atom$.
% E.g., 
 %  and assume that \alert{$\overline{\overline{a}} = a$}.
 We assume that %$\BK$ is closed under this $overline$ operator, namely
 for every $\atom$ occurring in $\formula{}$ the set $\overline{\atom}$ is not empty. %\nb{changed a bit here}  
  %In our examples and proofs, given an atom $\atom$, 
   %This is a restricted form 
  %of negation in which reasoning can still be performed in polynomial time,
  %as opposed to full negation where . 
  Whenever we write $\overline{\atom}$ in a rule, we assume it refers 
%as a conjunct of a clause 
to a %fixed but arbitrary %to refer to 
%¤a 
representative of an atom in $\overline{\atom}$, which means ``not $p$''. % assumed to be a singleton set.
Having a symbol that represents the exclusion of an atom can be seen as a weak form 
of negating it which is computationally efficient, since we remain in the Horn logic.
%, where
%  reasoning can be performed in polynomial time.
%\nb{added}

%Our approach, in fact, only requires that the atoms involved in incoherences have 
%such representative. 
% need 
%This is useful for resolving incoherences, where  excluding an atom.

\begin{example}\label{ex:model}\upshape
%, one can model 
The background knowledge 
of the stakeholders in  Example~\ref{ex:consensus} can be modeled as:
\begin{equation*}
\begin{array}{l}
 ({\sf parentsAlert}\wedge {\sf policeCall})\rightarrow \bot, \quad %\
 ({\sf child}\wedge {\sf adult})\rightarrow \bot, \\
 ({\sf child}\wedge {\sf teen})\rightarrow \bot, \quad %\
 ({\sf supervised}\wedge {\sf unsupervised})\rightarrow \bot. 
 \end{array}
\end{equation*} 
In our notation, $\overline{{\sf child}}=\{{\sf adult}, {\sf teen}\}$.\hfill {\mbox{$\triangleleft$}}
\end{example}  
%   $\BK = \{ \alpha \wedge \beta \rightarrow \bot \mid \alpha, \beta \textrm{ are atoms}\}$. 

Given a Horn clause $\phi$ with $p\in \ant{\phi}$, we denote by $\phi^{q\setminus p}$
the result of replacing $p\in\ant{\phi}$ by    $q$. % (if $p=\overline{q}$ then $\overline{p}=\overline{\overline{q}}=q$). 
Also, 
we denote by $\phi^{-p}$ the result of %Horn clause that results from
 removing $p$ from $\ant{\phi}$; and we denote by $\phi^{+p}$ the result of adding $p$ to $\ant{\phi}$.  
%\nb{Ana: added and changed the definition of conflict below}
%Given a Horn expression $\formula{}$ and a set of atoms $\sigma$, we define  $\sigma^*_{\formula{}}:=\{p\mid \formula{}\cup  \sigma \models p\}$.

%~ \begin{proposition}\label{prop:technical}
%~ Let $\formula{}$ be a definite Horn expression and let $\sigma$ be a conjunction of propositional symbols.
%~ The interpretation $\Imc$ with ${\sf true}(\Imc)=\sigma^*_{\formula{}}$ satisfies 
%~ $\formula{}$.
%~ \end{proposition}
 
\subsection{Coherence and Conflict} 
 %\textcolor{red}{% I think we need to be careful with "complement" to not confuse with negation, constraint seems safer
% As mentioned, we consider complementary atoms in the background knowledge. An atom and its complement are a 
% logically consistent set of formulas. We need to represent in logic semantics that an entity cannot have two complementary properties at the same time. For example, 
% a human should not be a both a child and an adult at the same time, choosing an action should mean forgoing the alternative available action. 
 To capture the %effect of the constraints in 
 semantics of the background knowledge on   behaviour rules of   stakeholders, we define two concepts: \emph{coherence} and \emph{conflict}. 
 Coherence is a property of a %pair of 
 definite Horn clause % defined with respect to 
 w.r.t. a   Horn expression or of a Horn expression. 
 Intuitively, in a ``well behaved'' set of rules, one should not be able to %, under the same holding conditions, 
 infer an atom $\atom$
 and an element of $\overline{\atom}$. 
 Such Horn expressions will be called \emph{incoherent}.  
  More specifically, a clause $\phi$ is incoherent with a 
 Horn expression when adding the antecedent of this clause to the set makes it possible for both the clause's consequent, $\cons{\phi}$, 
 and an element of %its complement 
 $\overline{\cons{\phi}}$  to be inferred from this union (see Example~\ref{ex:coh}).  
% We give an example.
 \begin{example}\label{ex:coh}\upshape
  Assume that 
  \begin{equation*}
\begin{array}{l}
  ({\sf parentsAlert}\wedge {\sf policeCall})\rightarrow \bot\in\BK. 
  \end{array}
 \end{equation*}
  Consider   the  Horn expression  $\{ {\sf illegalActivity}\rightarrow {\sf policeCall} \}$ and the 
  clause $({\sf illegalActivity}\wedge {\sf child})\rightarrow {\sf parentsAlert}$. The union of the antecedent of the clause, the clause itself, 
  and the Horn expression    implies %two complementary atoms to be inferred 
  both ${\sf policeCall}$ and  ${\sf parentsAlert}$, which should not happen according to   \BK. In symbols, 
 (${\sf parentsAlert}\in\overline{{\sf policeCall}}$ and vice-versa). \hfill {\mbox{$\triangleleft$}}
  \end{example}
  Having an  incoherent set of rules means that there might be a situation in which the set would not be able to offer any guidance as to what to do. 
  %}
  For simplicity, we may omit referring explicitly to the background knowledge %\BK 
since we assume one group of stakeholders at a time with a unique background knowledge. 
%\nb{moved the above sentence}

  Let $\phi$ and $\psi$ be definite Horn clauses and let $\formula{}$
be a definite Horn expression. 
A \emph{derivation of $\phi$ w.r.t. $\psi$ and $\formula{}$} is a sequence  
$\phi_1,\ldots,\phi_n$ of clauses in $\formula{}$ such that
\begin{itemize}
\item $\phi_1=\psi$, $\phi_n=\phi$,  
\item $\ant{\phi_{i+1}}\subseteq \bigcup_{1\leq j\leq i} \ant{\phi_j}\cup\cons{\phi_j}$,
and, 
\item for all $1 < i < n$, there is $j$ such that $i<j\leq n$ and $\cons{\phi_i}\in \ant{\phi_{j}}$. 
\end{itemize}
%It is \emph{minimal} if
%The \emph{derivation length} of $\phi$ w.r.t. $\psi$ and $\formula{}$
%is the minimal $n$ with $\phi_1,\ldots,\phi_n$ being 
%number of clauses in
%any such derivation. 
%$n$ such that $\phi_1,\ldots,\phi_n$ 
%Moreover, there is no other such sequence $\phi_1,\ldots,\phi_m$ %of $\phi$ w.r.t. $\psi$ and $\formula{}$ 
%with $m < n$. 
%\end{definition}
%
%
%Given two clauses $\psi,\phi$ in a Horn expression $\formula{}$,
We write $\psi\Rightarrow_{\formula{}} \phi$ if there is a derivation  
of $\phi$ w.r.t. $\psi$ and $\formula{}$.
Assuming $\phi,\psi\in\formula{}$, we have that
\[
\psi\Rightarrow_{\formula{}} \phi \text{ iff }
\formula{}\cup\ant{\psi}\models\ant{\phi}.\]
Some examples of derivations can be seen in Figure~\ref{fig:graph}.

The following proposition is useful for proving  our results\footnote{All of our proofs are given in detail in the Appendix}.  
%Proposition~\ref{lem:derivation} is an easy consequence of the notion of derivation
%definition 
%of the notation $\psi\Rightarrow_{\formula{}} \phi$.
% follows from 
%and the well-known fact that
%entailment in propositional Horn can be decided in polynomial time.
\begin{restatable}{proposition}{derivation}
\label{lem:derivation}
Given definite Horn clauses $\phi$ and $\psi$ and a definite Horn expression $\formula{}$, 
one can decide in %\nbtodo{
linear time on the number of literals in $\formula{}$ %?}polynomial time 
whether there is a derivation of $\phi$ w.r.t. $\psi$ and $\formula{}$. 
%and, if so,  compute the length of a minimal one in polynomial time.
\end{restatable} 

We are now ready for our definition of coherence.
%  \todo{explain $\Rightarrow_{\formula{}}$ notation}
 \begin{definition}[Coherence]\label{ex:coherence}  
%Let $\formula{}$ be a definite 
%Horn expression.
A definite Horn clause $\phi$ is  { \bf coherent  } with
a  
Horn expression ${\formula{}}$
if $\formula{}\setminus\{\phi\}\not\models \phi$ and 
\begin{itemize}
\item there is no $\psi\in\formula{}$ such that %\linebreak 
$\psi\Rightarrow_{\formula{}}\phi$ or $\phi\Rightarrow_{\formula{}}\psi$ while
%$\formula{}\cup\{\ant{\phi}\}\not\models
%p\wedge \cons{\psi}$ with $p\in\overline{\cons{\psi}}$,
%and $\formula{}\cup\{\ant{\psi}\}\not\models
%p\wedge \cons{\phi}$ with 
$\cons{\psi}\in \overline{\cons{\phi}}$ (note that $\cons{\psi}\in \overline{\cons{\phi}}$    implies 
$\cons{\phi}\in \overline{\cons{\psi}}$).
\end{itemize}
The set \formula{} is {\bf{coherent}} if all 
  $\phi\in\formula{}$ are coherent with $\formula{}$ (and \emph{incoherent} otherwise).
\end{definition}

 %\textcolor{red}{ 
 Conflict is a property of a set of definite Horn clauses. 
 The notion of  conflict is stronger than the notion of coherence. Intuitively, a
 conflict is an incoherence that cannot be easily resolved. 
% Consider . 
 The incoherence in Example~\ref{ex:coh}
 is due to the fact that ${\sf illegalActivity}\wedge {\sf child}$ implies
 both ${\sf parentsAlert}$ and ${\sf policeCall}$ while 
the background knowledge states that both cannot be true. 
 %{ex:coherence} with the background knowledge in
 %Example~\ref{ex:model}
 These rules are \emph{not} in conflict and incoherence can   be resolved as follows. 
 %Observe that the antecedent ${\sf illegalActivity}$ is implied by  the antecedent  
 %${\sf illegalActivity}\wedge {\sf child}$. 
 %If we assume that  ${\sf child}$ has a 
 %complement in the background knowledge, we can assume that 
 %Observe that 
 The rule  
 \[({\sf illegalActivity}\wedge {\sf child})\rightarrow {\sf parentsAlert}\] 
 can be considered as %a kind of 
 an exception to the more general rule  
 \[ {\sf illegalActivity}\rightarrow {\sf policeCall}.\] 
  %
% ${\sf parentsAlert}\in\overline{{\sf policeCall}}$ and vice-versa).
 Then, all we   have to do to restore coherence is to change the latter rule 
% $ {\sf illegalActivity}\rightarrow {\sf policeCall}$
  %${\sf illegalActivity}\wedge {\sf child}\rightarrow {\sf parentsAlert}$ 
  into a more specific 
  one that says: 
  %$ {\sf illegalActivity}\rightarrow {\sf policeCall}$
  \begin{itemize}
\item  unless the exceptional case (e.g., it is a child) has occurred,  take this action (e.g., call the police). 
  \end{itemize}
  A set can be in conflict 
 % for two main reasons.
 % The first is when there is a  ``double incoherence'': 
 % there are three clauses involved in the incoherence and resolving it would require choosing
 % between two or more rules. 
 %The second is 
 when we cannot find a ``suitable'' atom to add to the antecedent of an incoherent rule, % that is too general, 
  as a way to further specify it and avoid  incoherence. The ``suitable'' atoms are chosen from 
    the excludents of the atoms in the antecedent of the more specific rule involved in the incoherence.

\begin{definition}[Conflict]\label{def:conflict}  
Let $\formula{}$ be a definite 
Horn expression.
We say that $\formula{}$ is  {\bf in conflict} if
%~ when 
%~ there are $\phi_1,\phi_2,\psi\in\formula{}$
%~ s.t. $\phi_i\Rightarrow_{\formula{}}\psi$
%~ and $\cons{\phi_i}\in\overline{\cons{\psi}}$
%~ with $i\in\{1,2\}$
%$\formula{}\cup\ant{\phi}\models p\wedge \cons{\psi}$ with $p\in\overline{\cons{\psi}}$ 
%(i.e., $\formula{}$ has a double incoherence) or the following holds
\begin{itemize}
\item there are $\phi,\psi\in\formula{}$
s.t. $\phi\Rightarrow_{\formula{}}\psi$
and $\cons{\phi}\in\overline{\cons{\psi}}$
%$\formula{}\cup\ant{\phi}\models p\wedge \cons{\psi}$ with $p\in\overline{\cons{\psi}}$ 
(i.e., $\formula{}$ is incoherent); and 
\item there is no $r\in\ant{\phi}\setminus\ant{\psi}$ with $q\in \overline {r}$
s.t. $\psi^{+q}$ is coherent with $\formula{}\setminus\{\psi\}$.
\end{itemize}
\end{definition}

\begin{example}\label{ex:bk2}\upshape 
%Assume ${\sf parentsAlert}\wedge {\sf policeCall}\rightarrow \bot$
%is in the background knowledge $\BK$.
Consider \BK in Example~\ref{ex:coh} and the  rules:
\begin{equation*}
\begin{array}{ll}
(1) & {\sf illegalActivity} \rightarrow {\sf \sf parentsAlert}, \\
(2) &  {\sf illegalActivity} \rightarrow {\sf policeCall},\\
(3) & ({\sf illegalActivity} \wedge {\sf child})\rightarrow {\sf \sf parentsAlert}. \\
 \end{array}
 \end{equation*}
%\begin{equation}
%\begin{array}{l}
% \phi: { \sf emergency}\rightarrow \overline{\sf ChargeBattery}, \\
%\psi: {\sf night}\rightarrow {\sf ChargeBattery}. \\
% \end{array}
% \end{equation}
The set with the first two rules, (1) and (2), is in conflict, while the set  with the last two rules, (2) and(3),  is not. 
%In the next section, we give another (slightly more complex) example of rules in conflict in the proof of Theorem~\ref{ex:conflict}.
\hfill {\mbox{$\triangleleft$}}
% However, nothing excludes that $\ant(\phi)$ and $\and(\psi)$ happen at the same time, so this set $\{\phi, \psi\}$ is potentially in conflict?  
\end{example}
%We   assume that %the behaviour rules of 
%the stakeholders are not  in conflict with each other. 
%That is, $\bigcup^n_{i=1} \formula{i}$ is not in conflict (but may be incoherent).  
%that is if $\phi\in \formula{i}$ then there is no $\phi'\in \formula{j}$ such that 
%$\ant{\phi}=\ant{\phi'}$ and $\cons{\phi}=\overline{\cons{\phi'}}$. 
%Lastly, we assume that the stakeholders are individually coherent. 
%As usual, sets are incomparable if one is not contained in the other (and vice-versa). 
%~ Given definite Horn clauses $\phi,\psi$, we write $\phi\difference\psi$ if there is
%~ a definite Horn clause $\phi'$ such that $\ant{\phi'}$ is  incomparable with 
 %~ $\ant{\phi}$ or $\ant{\psi}$, there is no symbol $p$ 
 %~ such that $\{p,\overline{p}\}\subseteq \ant{\phi'}$ and $\ant{\phi'}\subseteq \ant{\phi}\cup\ant{\psi}$.
%~ %there is  $p \in \ant{\phi}  \oplus \ant{\psi}$ s.t. $\overline{p} \not\in  \ant{\phi}  \oplus \ant{\psi} $.
%~ If this is not the case then we write $\phi\notdifference\psi$.  \nb{one more attempt to get the right notion...}
\subsection{Common Ground} 
We are now ready to provide the notion of a common ground.
\begin{definition}[Common Ground]\label{def:consensus}
Let $\formula{1},\ldots \formula{n}$   be  definite Horn expressions, each associated with a stakeholder 
$i\in\{1,\ldots,n\}$.
%, %coherent and non-redundant, %\nb{A: maybe we dont need these cond}
%with  $\bigcup^n_{i=1} \formula{i}$  not redundant and not in conflict. 
Let  $\BK$ be a set describing  
background knowledge. A formula \formula{} is a {\bf{common ground}} for $\formula{1},\ldots, \formula{n}$ and \BK
if it %is non-redundant and 
satisfies each of the following postulates: 
\begin{itemize}[leftmargin=7.5mm]
\item[(P1)]  \formula{} is coherent; % \nb{O: changed} %consistent.
%\item[(P1)] if $\BK$ is satisfiable, then $\formula{}  \cup \BK$ is satisfiable;
\item[(P2)] if  $\bigcup^n_{i=1} \formula{i}$ is coherent, then $\formula{}  \equiv \bigcup^n_{i=1} \formula{i}$; 
\item[(P3)] for all $i \in\{1,\ldots,n\}$ and all $\phi \in \formula{i}$, we have that %\linebreak
 $\formula{}\not\models \ant{\phi} \rightarrow p$ with $p\in\overline{\cons{\phi}}$; %If $\formula{1} \cup \formula{2} \models \bot$, then $\formula{} \cup \formula{1} \not\models \bot$ iff $\formula{} \cup \formula{2}  \not\models \bot$ \nb{O: only with bk we can entail $\bot$}
\item[(P4)] for each $\phi\in\formula{}$, there is $\psi\in\bigcup^n_{i=1} \formula{i}$ with $\{\psi\}\models\phi$;
%~ \nb{M: Old P5 excluded $\{p \wedge q \rightarrow r, p \wedge \overline{q} \rightarrow r \}\subseteq \formula{}$. New P5 excludes $\{p \wedge q \rightarrow r, p \rightarrow r \}\subseteq \formula{}$.}
%\item[(P5)] for each $i\in\{1,\ldots,n\}$, there is (a non-trivial) $\phi\in\formula{}$ such that $\formula{i}\models\phi$;
\item[(P5)] {for all $i \in\{1,\ldots,n\}$ and all $\phi \in \formula{i}$,
%for each $i\in\{1,\ldots,n\}$, 
there is (a non-trivial) $\psi\in\formula{}$ such that $\{\phi\}\models\psi$; and}
%\item[(P6)] for all  $\phi\in\formula{}$, if there is $p\in \ant{\phi}$ such that, for all $q\in\overline{p}$, % we have that 
%$\formula{}\cup\{\phi^{q\setminus p}\}$ is coherent
%\blu{and
%there is $i\in\{1,\ldots,n\}$ such that $\formula{i}\models\phi^{q\setminus p}$}
%then there is $i\in\{1,\ldots,n\}$
%such that $\formula{i}\models\phi$ and $\formula{i}\not\models\phi^{-p}$.
%\blu{\item[(P6)] for all $\psi\in\formula{}$, if there exists a $\formula{i}$ and $\phi$ s.t. $\formula{i}\models \phi$, and $\phi\models \psi$, but $\formula{} \nvDash \phi$,   then  $\phi $ is incoherent with $\bigcup^n_{i=1} \formula{i}$. }  
%~ \end{itemize}
%~ A Horn expression \formula{} is a {\bf{full compromise}}
%~ for the formulas $\formula{1},\ldots \formula{n}$ and \BK if
%~ \formula{} is a  compromise    for $\formula{1},\ldots \formula{n}$ and
%~ \begin{itemize}[leftmargin=7.5mm]
\item[(P6)] for all  $\phi\in\formula{}$, if there is $p\in \ant{\phi}$ such that, for all $q\in\overline{p}$, % we have that 
$\formula{}\cup\{\phi^{q\setminus p}\}$ is coherent {and
there is $i\in\{1,\ldots,n\}$ such that $\formula{i}\models\phi^{q\setminus p}$
then $\formula{i}\not\models\phi^{-p}$.}
%\item[(P7)]
%there is no clause $\phi$ such that $\formula{}\not\models \phi$ and 
%$\formula{}\cup\{\phi\}$ is a 
%compromise that respects (P6).    
%for all possible compromise  formulas $\formula{}'$
%such that $\formula{}'\models \formula{}$ %and 
%we have that $\formula{}\models \formula{}'$. 
\end{itemize}
%We may refer to this last property as (P7).
\end{definition}

%Let us intuitively explain and justify 
We now discuss and motivate the postulates that characterize a %(full) 
 common ground. 
%(P0) and 
\begin{description}

\item[(P1)] The first postulate is intuitive: the learned set of rules should be 
coherent %and  consistent 
with the background knowledge. If they were not so, the theory 
would recommend, for example, two mutually exclusive courses of action for the same 
situation (described here in terms of rule antecedents).   

\item[(P2)] %We would like that 
The   common ground should have as much of the 
input rules supplied by the stakeholders as possible, hence (P2) ensures that if the union of  
rules provided by stakeholders is  coherent, then this should be the common ground.
% should be 
%their union. % of the individual recommendations.  

\item[(P3)] The motivation for (P3) we find in \cite{Hare1972}: ``the essence of morality is to treat the interests of others as of equal weight with ones own'',  which we here interpret as a requirement that all agent's rules are considered equally informative and 
should not be entirely overridden.  (P3) ensures that a rule that is in strict opposition with what a 
stakeholder recommends is not in  the common ground.    

\item[(P4)]  We also do not want that some rules ``sneak in'' % I think sneak up is on a person and sneak in is on some place
in the common ground, 
without being explicitly supported by a stakeholder. This is operationalized by 
(P4) that guarantees that a rule in a common ground can always be ``traced back'' to a rule from a %at least one %originating  
stakeholder. 

\item[(P5)] The fifth postulate  ensures that some 
part of a stakeholder's rule is in a common ground, though, in a  ``weaker'' form.  
In other words, a non-trival part of each stakeholder's rules   should be in a common ground. 
%
%Lastly, we have 
%<<<<<<< HEAD
%(P6)  is the most ``tricky'' postulate to explain, however, it is 
%essential for %the Ps. no need for 'the' here
%Definition~\ref{def:consensus} because it
%=======
\item[(P6)] The sixth is the most ``tricky'' postulate to explain, however, it is 
essential for %the Ps. no need for 'the' here 
Definition~\ref{def:consensus} because it
%>>>>>>> b030ddd65fc84acfa3547dbd0a4cdec0a3ced38f
%to ensure that redundancies 
avoids that unintended rules become part of the  common ground. 
%~ Note that (P5) allows for a subset of rules such as,  for example
%~ % $\{p \wedge q\rightarrow r, \overline{p} \wedge  q \rightarrow r \}$. 
  %~ $\{p_1 \wedge p_2\rightarrow q, \overline{p_1} \wedge  p_2 \rightarrow q \}$. 
 %~ These two can only be replaced by 
 %~ %$\{ q \rightarrow r\}$ 
 %~ $\{ p_2 \rightarrow q\}$ 
 %~ if there is an agent who has this rule 
 %~ %$ q \rightarrow r$, 
 %~ $p_2 \rightarrow q$, 
 %~ \textcolor{red}{but also such as for example} \todo{Not clear what "also such as" refers to here}
%~ % $p \wedge  s \rightarrow \overline{r}$.  
  %~ $p_1 \wedge  p_3 \rightarrow \overline{q}$.  
 %~ But (P6) does more than this. Assume we have  
%~ % $\formula{1}=\{p  \rightarrow r\}$ 
 %~ $\formula{1}=\{p_1  \rightarrow q_1\}$  
 %~ and 
 %~ %$\formula{2} =\{ p \wedge  q \rightarrow \overline{r}, s \rightarrow t \}$.
  %~ $\formula{2} =\{ p_1 \wedge  p_2 \rightarrow \overline{q_1}, p_3\rightarrow q_2 \}$.
  %~ The set 
  %~ %$\{ p\wedge \overline{q} \rightarrow r, p \wedge  q \rightarrow \overline{r}, s \rightarrow t \}$ 
  %~ $\{ p_1\wedge \mathbf{\overline{p_2}} \rightarrow q_1, p_1 \wedge  p_2 \rightarrow \overline{q_1}, p_3 \rightarrow q_2 \}$ 
  %~ satisfies (P6) and is a possible compromise, but not the set 
%~ %  $\{ p\wedge s\rightarrow r, p \wedge  q \rightarrow \overline{r}, s \rightarrow t \}$.  
%~ $\{ p_1\wedge \mathbf{p_3}\rightarrow q_1, p_1 \wedge  q_2 \rightarrow \overline{q_1}, p_3 \rightarrow q_2 \}$.  
We illustrate this with the following example. 
%~ \begin{example}\upshape
%~ Consider %the Horn expressions
%~ \[
%~ \begin{array}{l}
%~ \formula{1}=\{p\rightarrow q, \ s\rightarrow t\}\text{ and } \formula{2}=\{(p\wedge r)\rightarrow \overline{q}\}.
%~ \end{array} 
%~ \]
%~ %$$
%~ %$$.  
%~ %
%~ To resolve the incoherence in $\formula{1}\cup\formula{2}$,   
%~ %$p\rightarrow q$ with \formula{},
%~ one can replace $p\rightarrow q$ with $p\wedge \overline{r}\rightarrow q$.
%~ Without (P6), 
%~ %the resulting Horn expression 
%~ \[
%~ \begin{array}{l}
%~ \formula{}=\{p\wedge \overline{r}\rightarrow q, \ s\rightarrow t, \ (p\wedge r)\rightarrow \overline{q}\}
%~ \end{array} 
%~ \]
 %~ would {\bf{not}}
%~ be a compromise for $\formula{1}\cup\formula{2}$. The reason is because
%~ it would violate (P7) since %there is
%~ %$\formula{}'=
%~ \[
%~ \begin{array}{l}
%~ \formula{}\cup\{p\wedge s\rightarrow q, \ p\wedge t\rightarrow q\} 
%~ \end{array} 
%~ \]
%~ %$$
%~ %such that $\formula{}'\models \formula{}$, 
%~ %$\formula{}\not\models \formula{}'$ and
%~ %$\formula{}'$ 
%~ %that 
%~ satisfies all postulates except for (P6). 
%~ The rules $p\wedge s\rightarrow q, \ p\wedge t\rightarrow q$
%~ are unintended because $s,t$ are  unrelated with the incoherence in  $\formula{1}\cup\formula{2}$. \hfill {\mbox{$\triangleleft$}}
%~ %belong to a rule
%~ %not involved in the derivation of atom $q$.
%~ \end{example}
\begin{example}\label{ex:psix}\upshape
Consider \BK in Example~\ref{ex:coh} and %the Horn expressions
\[
\begin{array}{l}
\formula{1}=\{\phi={\sf illegalActivity} \rightarrow {\sf policeCall}, \\
 \quad\quad\quad \psi={\sf lowBattery} \rightarrow {\sf charge}\}\\
%\text{ and } 
\formula{2}=\{\varphi=({\sf illegalActivity} \wedge {\sf child})\rightarrow {\sf \sf parentsAlert}\}.
\end{array} 
\]
%$$
%$$.  
%
To resolve the incoherence in $\formula{1}\cup\formula{2}$, % (w.r.t. \BK),   
%$p\rightarrow q$ with \formula{},
one can replace $\phi$ % $p\rightarrow q$ 
with $({\sf illegalActivity} \wedge \overline{{\sf child}})\rightarrow {\sf \sf policeCall}$.
Without (P6), 
the rules
\[
\begin{array}{l}
({\sf illegalActivity}\wedge {\sf lowBattery})\rightarrow {\sf policeCall},\\
 \ ({\sf illegalActivity}\wedge {\sf charge})\rightarrow {\sf policeCall}
\end{array} 
\]
%the resulting Horn expression 
%\[
%\begin{array}{l}
%\formula{}=\{\phi', \ \psi, \ \varphi\}
%\end{array} 
%\]
% would {\bf{not}}
%be a full compromise for $\formula{1}\cup\formula{2}$. The reason is because
%it would violate (P7) since %there is
%$\formula{}'=
%~ \[
%~ \begin{array}{l}
%~ \formula{}\cup\{({\sf illegalActivity}\wedge {\sf lowBattery})\rightarrow {\sf policeCall},\\
 %~ \ ({\sf illegalActivity}\wedge {\sf charge})\rightarrow {\sf policeCall}\} 
%~ \end{array} 
%~ \]
%$$
%such that $\formula{}'\models \formula{}$, 
%$\formula{}\not\models \formula{}'$ and
%$\formula{}'$ 
%that 
could also be used to replace $\phi$ as they satisfy (P1)-(P5). %satisfies all postulates except for (P6). 
Though, these rules %$p\wedge s\rightarrow q, \ p\wedge t\rightarrow q$
are unintended since ${\sf lowBattery}$ and ${\sf charge}$ are  unrelated with the incoherence in  $\formula{1}\cup\formula{2}$. \hfill {\mbox{$\triangleleft$}}
%belong to a rule
%not involved in the derivation of atom $q$.
\end{example}

% (P5) is not as obvious as the other postulates.al We include it to ensure that redundancy \nb{A: I think we may need to specify which kind of redundancy we are talking about ... because it is not the notion of redundancy used above (which we should also specify.... ) }
% is not kept, but also when a recommendation is made more specific to accommodate another recommendation, this is done using ``meaningful" atoms. 
% Consider the incoherent set $\{ p \rightarrow \overline{r}, (p\wedge q) \rightarrow r, s \rightarrow t\}$. We can make this set coherent by making the first recommendation ``more specific"  transforming it into $\{ (p \wedge \overline{q}) \rightarrow \overline{r}, (p\wedge q) \rightarrow r, s \rightarrow t\}$, but we can accomplish the same thing with the set $\{ (p \wedge \overline{s}) \rightarrow \overline{r}, (p\wedge q) \rightarrow r, s \rightarrow t\}$. (P5) us here to enforce the first choice is taken. \nb{M: I am still not sure (P5) does this.}
%
 %~ \item[(P7)]Our postulates try to capture a principle of minimal change  but,  at the same time, we want 
%~ to build the most informative behaviour recommendation possible. This   is reflected in (P7). % in Definition~\ref{def:consensus}. 
\end{description}
%The idea to describing desirable properties of a representative set, 
%or an agreement set, as postulates we borrow from the belief revision and 
%merging literature, as for example ~\cite{KoniecznyP11}. 
%Belief merging postulates try to 
%capture the principle of minimal change when a belief base is updated or belief 
%bases are merged. We are trying to do the same,
%
 
%\section{Learning Moral Recommendations}\label{sec:learning-algorithms}

 \section{Finding Common Grounds}\label{sec:building}
 
 We investigate the problem of finding a common ground 
for  rules supplied by  stakeholders, considering 
that they have a basic background knowledge, as  described in Section~\ref{sec:definitions}.
%In our results, 
We use the notion 
of \emph{non-redundant} and \emph{acyclic} Horn expressions, defined as follows.
A Horn expression  \formula{} is  {non-redundant} 
if for all $\phi\in\formula{}$ it is not the case that  $\formula{}\setminus\{\phi\}\models \phi$. 
%We say that 
%A Horn expression $\formula{}$ 
It is acyclic
if there is no sequence of clauses $\phi_1, \ldots, \phi_n \in \formula{}$ 
such that $\cons{\phi_i}\in \ant{\phi_{i+1}}$, for all $1\leq i < n$,
and $\phi_1=\phi_n$. 

In particular, we  show that
for 
%\begin{itemize}
%\item 
\textbf{non-redundant}, 
%\item 
\textbf{not in conflict}, and 
%\item 
\textbf{acyclic 
Horn expressions}, 
%\end{itemize}
a  common ground is guaranteed to exist
and can be computed in {polynomial time }(Theorem~\ref{thm:main}). 
%We prove our result by presenting an algorithm (Algorithm~\ref{alg:consensus}) that computes a  common ground
%if the mentioned conditions are met. 
Our result is tight in the sense that if any of these three conditions is removed then
%We show that, in general, 
%full 
common grounds may not exist. 
%We leave open the question of whether full 
%compromises can be computed in polynomial time
%under the three conditions. 
We first show the negative results, stated in Theorem~\ref{thm:main-negative}.

\begin{restatable}{theorem}{mainnegative}
\label{thm:main-negative}
Consider the class of
    non-redundant, not in conflict, and acyclic  Horn expressions.  %$\formula{1}, \ldots, \formula{n}$ 
%  such that
%$\bigcup^n_{i=1}\formula{i}$ is in conflict and 
If we extend this class by removing any of the three conditions 
(while still keeping the remaining two)
%there are 
%for which no %full 
a common ground for  Horn expressions in the extended class may not exist.
%
%There are  Horn expressions $\formula{1}, \ldots, \formula{n}$  %with a background theory $\BK$ 
%such that $\bigcup^n_{i=1}\formula{i}$ is not in conflict
%and is not redundant but no compromise for $\formula{1}, \ldots, \formula{n}$ %and $\BK$ 
%exists. 
\end{restatable}

The rest of this section is devoted to show that if   $\bigcup^n_{i=1}\formula{i}$
is an acylic, non-redundant, and not in conflict Horn expression then a common ground always exists (Theorem~\ref{thm:main}).
Our proof strategy consists in showing
that Algorithm~\ref{alg:consensus}  returns a common ground for Horn expressions $\formula{1}, \ldots, \formula{n}$,
 if the three mentioned conditions are satisfied.
 % $\bigcup^n_{i=1}\formula{i}$ is acylic.
We also show that Algorithm~\ref{alg:consensus} terminates in 
%\nbtodo{polynomial 
%time} 
in the size of $\formula{1}, \ldots, \formula{n}$ and the size of the background knowledge. %, given as input.

Before we explain the algorithm, we introduce some notions.
Common grounds are  found by modifying incoherent clauses, in particular,
by adding atoms to the antecedent of a clause $\phi$. As we shall see later, at most one atom is added.
The resulting clause $\phi'$ is such that  $\{\phi\}\models \phi'$ but $\{\phi'\}\not\models \phi$.
We may refer to $\phi'$ as the result of `weakening' %a `weakened' version of 
$\phi$ by adding some atom  to its antecedent, or simply say that
%. We may omit 
%the information 
%of 
%which atom was added if this is irrelevant and simply say that
$\phi'$ is a `weaker' version of $\phi$. 
%
%\begin{definition}[Safety and Dependency Graph]
%\label{def:minimal}
%~ A \emph{derivation of $\phi$ w.r.t. $\psi$ and $\formula{}$} is a sequence  
%~ $\phi_1,\ldots,\phi_n$ of clauses in $\formula{}$ such that
%~ $\phi_1=\psi$, $\phi_n=\phi$,  
%~ $\ant{\phi_{i+1}}\subseteq \bigcup_{1\leq j\leq i} \ant{\phi_j}\cup\cons{\phi_j}$,
%~ and, for all $1 < i < n$, there is $j$ such that $i<j\leq n$ and $\cons{\phi_i}\in \ant{\phi_{j}}$. 
%It is \emph{minimal} if
%The \emph{derivation length} of $\phi$ w.r.t. $\psi$ and $\formula{}$
%is the minimal $n$ with $\phi_1,\ldots,\phi_n$ being 
%number of clauses in
%any such derivation. 
%$n$ such that $\phi_1,\ldots,\phi_n$ 
%Moreover, there is no other such sequence $\phi_1,\ldots,\phi_m$ %of $\phi$ w.r.t. $\psi$ and $\formula{}$ 
%with $m < n$. 
%\end{definition}
%
%
%Given two clauses $\psi,\phi$ in a Horn expression $\formula{}$,
%~ We write $\psi\Rightarrow_{\formula{}} \phi$ if there is a derivation  
%~ of $\phi$ w.r.t. $\psi$ and $\formula{}$. 

%
%Given a Horn expression $\formula{}$, 
\begin{definition}\label{def:safe} Let $\phi$ and $\psi$ be definite Horn clauses and let $\formula{}$
be a definite Horn expression\footnote{We may omit `for $\formula{}$'
if this is clear from the context.}. The (incoherence) %{\textbf{
{\bf\dependencygraph 
of %a Horn expression
 $\mathbf{\formula{}}$} is the directed graph $(V,E)$, where
\begin{itemize}
\item $V$ is the set of all pairs $(\psi,\phi)$ such that 
$\psi\Rightarrow_{\formula{}} \phi$ and $\cons{\psi}\in \overline{\cons{\phi}}$, and,
\item $E$ %\subseteq V\times V$ 
is the set of all  $((\psi',\phi'),(\psi,\phi))$ such that
$\phi'\neq \phi$ and 
$\phi'$ occurs in a derivation of $\phi$ w.r.t. $\psi$ and $\formula{}$.
\end{itemize}
%$\psi$ is coherent with the Horn expression that results from replacing 
%$\phi'$ in $\formula{}$ by a weaker version of   $\phi'$. 
%(by adding some atom to the antecedent 
%of $\phi'$).
We say that $v'\in V$ is a \em{parent} for $v\in V$
if % there is no $v'\in V$ such that 
$(v',v)\in E$. \upshape

We 
say that the pair $(\psi,\phi)$ is %{\textbf{
\textbf{\safe} for $\formula{}$ if 
$(\psi,\phi)$ has no parent in the dependency graph
of $\formula{}$.
%\begin{itemize}
%\item 
%~ there  
%~ is a derivation $\phi_1,\ldots,\phi_n$
%~ of $\phi$ w.r.t. $\psi$ and $\formula{}$
%~ such that %, , we have that
%~ $\phi_i$ is coherent with $\formula{}$, for all $1< i <n$, and
%~ %\nb{the next bit may change}
%~ $\cons{\psi}\in \overline{\cons{\phi}}$.
%\end{itemize}
\end{definition} 

Figure~\ref{fig:graph} illustrates an example where there is no pair of \safe clauses.%\footnote{This is the dependency graph of the clauses in the proof of Theorem~\ref{thm:non-existence}.}% Marija: theorem 18 does not show up  until the Apendix so it is confusing to mention it here
For this example, one can also see that there is no common ground (cf. Definition~\ref{def:consensus})
for the Horn expressions. %in Figure~\ref{fig:graph}.

By Lemma~\ref{lem:safe-exists} stated in the following, if the theory is incoherent 
and the Horn expression is acyclic (which also means avoiding cycles
in the dependency graph)
then there is a safe pair of clauses in it.  We use this property in 
Algorithm~\ref{alg:consensus}. 

%\end{definition}
\begin{figure}[t]
    \includegraphics[width=\linewidth]{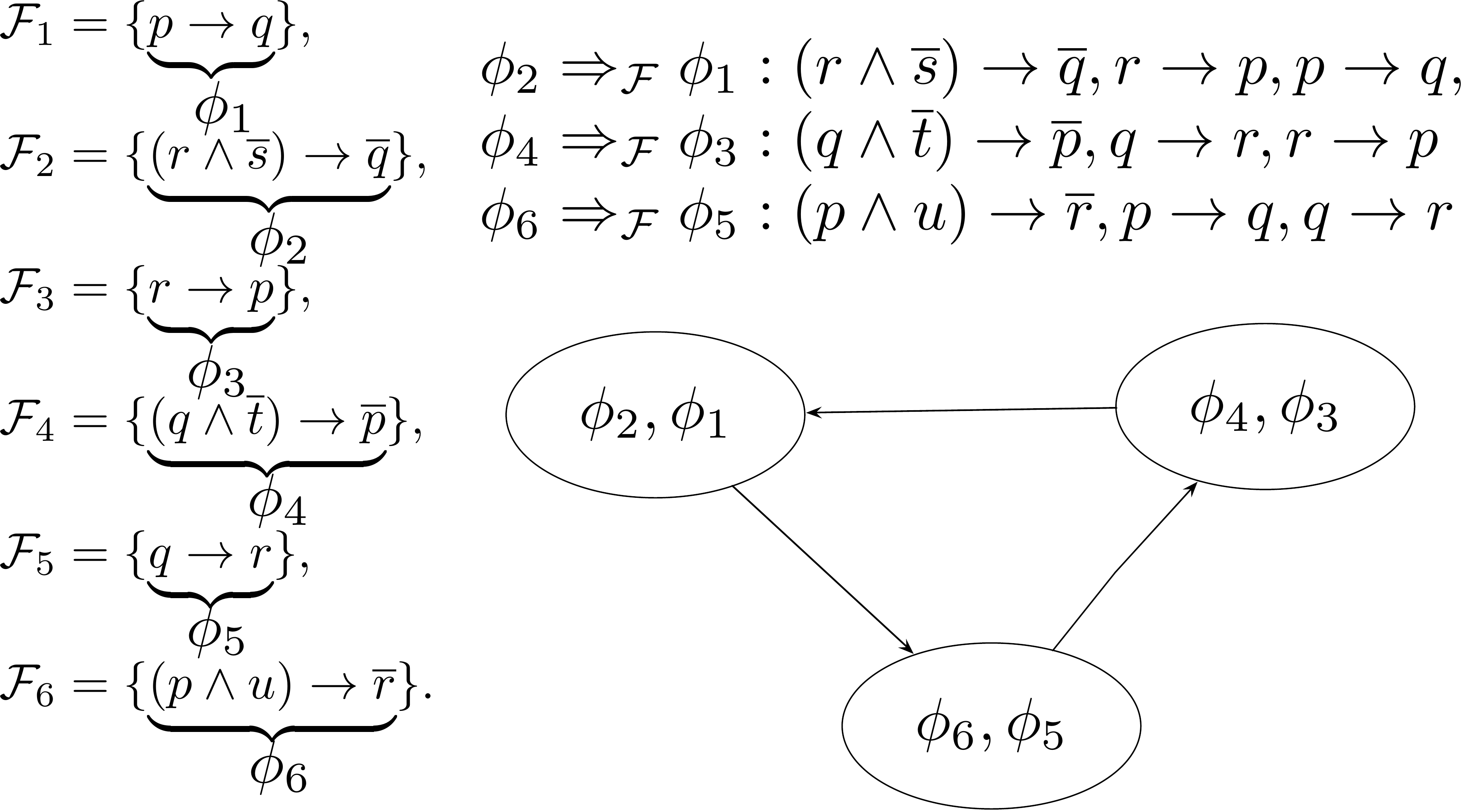}
    \caption{A dependency graph.}
    \label{fig:graph}
\end{figure}

\begin{restatable}{lemma}{safeexists}
\label{lem:safe-exists}
Let $\formula{}$ be an acyclic Horn expression. 
If $\formula{}$ is  incoherent then there are 
$\psi,\phi\in \formula{}$ such that $(\psi,\phi)$ is \safe (Definition~\ref{def:safe}).
Moreover, one can find $\psi,\phi\in \formula{}$ such that $(\psi,\phi)$ is \safe
in quadratic time in the size of $\formula{}$ (and \BK). 
\end{restatable}

\upshape
Algorithm~\ref{alg:consensus} receives as input the background knowledge \BK and a finite list of definite Horn expressions $\formula{1}, \ldots \formula{n}$. It verifies that $\formula{} = \ \bigcup^n_{i=1}\formula{i}$
%\nb{add here that the horn theory is restricted, actions and preconditions are disjoint}
 is not in conflict,  not redundant, and not acyclic.
%If $\formula{}$ is acyclic then 
At each iteration of the ``while'' loop (Line~5),
Algorithm~\ref{alg:consensus} 
first selects clauses $\psi,\phi\in\formula{}$ such that $(\psi,\phi)$ is safe (Line~6).
%$\formula{}\cup \ant{\psi}\models
%\atom\wedge \cons{\phi}$ with $\atom\in\overline{\cons{\phi}}$ 
%
By Lemma~\ref{lem:safe-exists}, at least one safe pair is guaranteed to exist.
 Then, in Line~7, it resolves  incoherences
%
%the algorithm selects a pair $(\psi,\phi)$
%of safe clauses (guaranteed to exist by Lemma~\ref{lem:safe-exists})
by replacing $\phi$ with
all weaker versions of this clause that are coherent with the Horn expression being constructed.
As we formally state later on in Theorem~\ref{thm:main}, Algorithm~\ref{alg:consensus}  outputs a common ground for $\formula{}$ (and \BK).

\begin{figure}[t]
    \includegraphics[width=\linewidth]{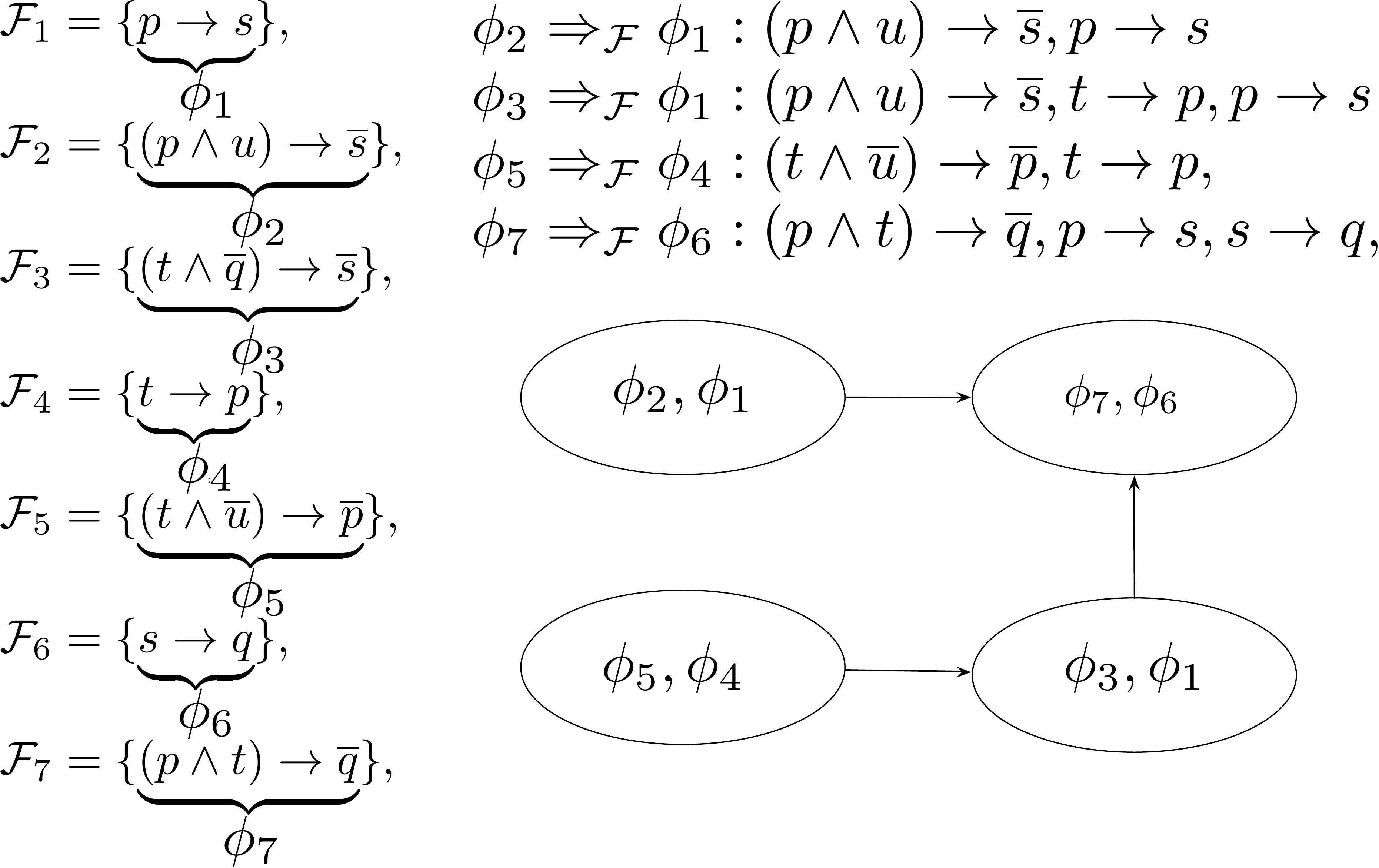}
    \caption{A dependency graph with \safe clauses.}
    \label{fig:graph2}
\end{figure}

\begin{example}\upshape 
Consider $\formula{1}-\formula{7}$ as in Figure~\ref{fig:graph2}. 
We have that $\formula{}=\bigcup^7_{i=1}\formula{i}$ is not in conflict, 
 not redundant, acyclic, but incoherent. The nodes $(\phi_2,\phi_1)$ and $(\phi_5,\phi_4)$ have no parent. 
 Algorithm~\ref{alg:consensus} iterates twice, each time selecting one of these pairs (the order does not change the result).
 It then returns a common ground $\formula{}^\ast$ for $\formula{1}, \ldots, \formula{7}$,  which is the union of: 
\[
\begin{array}{ll}
\formula{1}^*=\{(p \wedge \overline{u})\rightarrow s \}, & \formula{2}= \{ (p\wedge u)\rightarrow \overline{s} \},\\
%\\
\formula{3}=\{ (t\wedge \overline{q})\rightarrow \overline{s}\}, & \formula{4}^*=\{ (t \wedge u)\rightarrow p\},\\
%\\ 
\formula{5}=\{ (t\wedge \overline{u})\rightarrow \overline{p}\}, & \formula{6}=\{ s\rightarrow q\}, \\
%\\
\formula{7}=\{(t\wedge u)\rightarrow \overline{q}\}. &\\
\end{array}
\]
\hfill {\mbox{$\triangleleft$}}
\end{example}

We assume that Algorithm~\ref{alg:consensus}  selects  clauses following
some fixed but arbitrary order (e.g. lexicographic) if there are multiple \safe pairs of clauses.

\begin{algorithm}[tb]
\caption{Building coherent \formula{} }
\label{alg:consensus}
\textbf{Input}: Horn expression sets $\formula{1}, \ldots \formula{n}$ and $\BK$. \\%with $\formula{}$ being acyclic, non-redundant, and  %incoherent but 
%		not in conflict\\
%\textbf{Parameter}: Optional list of parameters\\
\textbf{Output}: A common ground  for $\formula{1}, \ldots \formula{n}$and $\BK$ or $\emptyset$

\begin{algorithmic}[1] %[1] enables line numbers
\STATE $\formula{}:= \formula{1} \cup \cdots \cup \formula{n}$
\IF {$\formula{}$ is  cyclic \textbf{or} redundant \textbf{or}    in conflict} \STATE  \textbf{return}  $\emptyset$ (A common ground may not exist by Th.~\ref{thm:main-negative}) \ENDIF
\WHILE{\formula{} is incoherent
%there are $\psi,\phi\in \formula{}$		with $\psi\Rightarrow_{\formula{}}\phi$
%$p\in\overline{\cons{\phi}}$ and 
%$\formula{}\cup \ant{\psi} \models
%p\wedge \cons{\phi}$
%and $\cons{\psi}\in\overline{\cons{\phi}}$
}
\STATE Find $\psi,\phi\in \formula{}$ such that $(\psi,\phi)$ is \safe
%and $\cons{\psi}\in \overline{\cons{\phi}}$ 
\label{l:choice}
\STATE 	Replace $\phi$ by all
$\phi'\in\{ \phi^{+p}\mid p\in \overline{\literal},
				 \literal\in\ant{\psi}\setminus\ant{\phi}\}$
				 coherent with 
				 $\formula{}\setminus\{\phi\}$ \label{ln:replace}
\ENDWHILE
\STATE \textbf{return} \formula{}
\end{algorithmic}
\end{algorithm}

The condition that 
$\formula{}\setminus\{\phi\}\not\models\phi$ in Definition~\ref{ex:coherence}
is important to ensure that if  the input of Algorithm~\ref{alg:consensus} is non-redundant then the output 
is also non-redundant. %avoid introducing redundancies while resolving incoherences.
We illustrate this with the following example.
\begin{example}\upshape
Consider %the formulas
\[
\begin{array}{l}
\formula{1}=\{(q \wedge r)\rightarrow s, \ (p\wedge q\wedge \overline{r})\rightarrow \overline{s}\}, \\
\formula{2}= \{ (p\wedge q)\rightarrow s \}.\\
%\formula{3}=\{ \}.\\
%\formula{4}=\{ (t \wedge u)\rightarrow p\},\\ 
%\formula{5}=\{ (t\wedge \overline{u})\rightarrow \overline{p}\},  \\
%\formula{6}=\{ s\rightarrow q\},\\
%\formula{7}=\{(t\wedge u)\rightarrow \overline{q}\}.\\
\end{array}
\]
Algorithm~\ref{alg:consensus} would resolve the incoherence in this case by 
replacing the clause  in  $\formula{2}$ by  $(p\wedge q\wedge {r})\rightarrow {s}$.
The latter rule is redundant because it is implied by $\formula{1}$.
 \end{example}

We point out that conflict, redundancy, and acyclicity %conflict 
are all
conditions that can be determined in polynomial time, %since it is may not be so clear that the cublic comes from quadradic times linear it may be better to just leave poly here
since the number of iterations is linear on the size of the input,  it can be determined in 
polynomial time whether the input of Algorithm~\ref{alg:consensus} is 
as expected.
We are now ready to state our main theorem. % (our upper bound).

\begin{restatable}{theorem}{main}
%\begin{theorem}
\label{thm:main}
The output of Algorithm~\ref{alg:consensus} with
an acyclic, non-redundant, not in conflict Horn expression $\bigcup^n_{i=1}\formula{i}$ and \BK as input
is a  common ground for $\bigcup^n_{i=1}\formula{i}$ and \BK. Algorithm~\ref{alg:consensus} terminates in 
$O((|\bigcup^n_{i=1}\formula{i}|+|\Bmc|)^4)$.
%\nbtodo{polynomial time} w.r.t.
%$|\bigcup^n_{i=1}\formula{i}|+|\BK|$.
%Then, \formula{} is a compromise
%for $\bigcup^n_{i=1}\formula{i}$. % (and \BK).
% if
%\formula{} is a possible compromise    for $\formula{1},\ldots \formula{n}$ and  for all possible compromise  formulas $\formula{}'$
%such that $\formula{}'\models \formula{}$ %and 
%we have that $\formula{}\models \formula{}'$. 
\end{restatable}

%~ It would be possible to modify Algorithm~\ref{alg:consensus} so as to ensure
%~ (P6) in
%~ %the last condition of 
%~ Definition~\ref{def:consensus} by removing clauses that violate this property
%~ before returning the output of the algorithm. 
%~ All the steps needed to check this property can be computed in polynomial time. 
%~ However, ensuring this postulate could potentially violate (P5).
%~ The main difficutly in building full compromises is the balance between
%~ these two properties.

%Even though we choose to ensure (P5) over (P6) in the algorithm,  
The \safe condition in Line~\ref{l:choice} of Algorithm~\ref{alg:consensus}
avoids that the algorithm introduces rules that
%in a derivation (Definition~\ref{def:minimal})
%could result in a Horn expression that 
violate (P6). We illustrate this in the following example. %~\ref{ex:psix} a case showing that %in which
%removing    %of minimality 
%, (P7).
\begin{example}\label{ex:psixtwo}\upshape
Assume that the \safe condition is not %present 
 in Line~\ref{l:choice} of Algorithm~\ref{alg:consensus} and that 
 the algorithm can choose any $\psi,\phi\in\formula{}$ such that  
 $\psi\Rightarrow_{\formula{}} \phi$ and $\cons{\psi}\in\overline{\cons{\phi}}$.
Let 
\[
\begin{array}{l}
\formula{}:=\{p\rightarrow s, \
  (t\wedge\overline{q})\rightarrow \overline{s},\
  t\rightarrow p,\
  (t\wedge \overline{u}) \rightarrow\overline{p} 
\}.
\end{array} 
\]
%~ Consider $\formula{}$ with 
%~ \begin{equation}
%~ \begin{array}{l}
%~ c_1: p\rightarrow s,\\ 
 %~ c_2: t\wedge\overline{q}\rightarrow \overline{s},\\ 
%~ c_3: t\rightarrow p,\\
%~ c_4:  t\wedge \overline{u} \rightarrow\overline{p} \\
%~ %  c_5: s\rightarrow q.
 %~ \end{array}
%~ \end{equation}
%where  clause $c_i$ is in  some $F_i$.
%Our algorithm could select $c_1$ and $c_2$ because the antecedent of
%$c_2$ implies the antecedent of $c_1$ and this clause has minimal derivation length such that
%the consequent of $c_1$ is in the complement of the consequent of $c_2$.
%Without the \safe condition,
Then, our algorithm could select the first two clauses
and replace $p\rightarrow s$ by some clauses, $(p\wedge q)\rightarrow s$ being one of them.
%Then it replaces $c_1$ by some clauses, $p\wedge q\rightarrow s$ being one of them.
 %There is still an incoherence between $c_3$ and $c_4$. Then, $c_3$ is replaced by
The last two clauses are still incoherent. Selecting them means 
%If the algorithm selects
  that $t \rightarrow p$ is replaced by
$(t\wedge u) \rightarrow p$. 
Now $(p\wedge q)\rightarrow s$
violates (P6) because $(p\wedge \overline{q})\rightarrow s$
is coherent with the other clauses and $\formula{}\models p\rightarrow s$ (in other words, 
$\formula{}\models ((p\wedge q)\rightarrow s)^{-q}$). \hfill {\mbox{$\triangleleft$}}
%but $p\wedge q \rightarrow s$
%does not satisfy P6.  
\end{example}

% Ensuring (P7) without (P6) could introduce 
%unintended clauses, such as the ones in Example~\ref{ex:psix}.
%the input of Algorithm~\ref{alg:consensus} is non-redundant
%then the output is also non-redundant. 

So far we have  referred to `a common ground'.
Example~\ref{ex:notunique} illustrates that common grounds may not be unique.
%thus avoiding non-determinism. 

\begin{example}\label{ex:notunique}\upshape
Consider %the formulas
\[
\begin{array}{l}
\formula{1}=\{(t \wedge \overline{p})\rightarrow r, \ (\overline{p}\wedge s)\rightarrow {r} \}, \  %\\
\formula{2}= \{ \overline{p}\rightarrow \overline{r} \}. %,  %\\
%\formula{3}=\{ \}.\\
%\formula{4}=\{ (t \wedge u)\rightarrow p\},\\ 
%\formula{5}=\{ (t\wedge \overline{u})\rightarrow \overline{p}\},  \\
%\formula{6}=\{ s\rightarrow q\},\\
%\formula{7}=\{(t\wedge u)\rightarrow \overline{q}\}.\\
\end{array}
\]
Replacing the clause in $\formula{2}$ by either $(\overline{p}\wedge\overline{t})\rightarrow \overline{r}$
or $(\overline{p}\wedge\overline{s})\rightarrow \overline{r}$ would yield a common ground.
\end{example}
 
A way to deal with non-uniqueness could be by adding all possible weakening options in the common ground. 
In the next sections, we relate our  work
with the literature and discuss our results. 
\section{Related Work} \label{sec:relwork}
%Within AI, the aggregation of opinions from multiple stakeholders has been considered as a problem of computational social choice, 
%specifically voting \citep{Zwicker16,NoothigattuGADR18},  or judgment aggregation \citep{Endriss16}. We can also consider 
%formulating the aggregation of opinions as a belief merging problem \citep{SchwindM18}. Belief merging (BM) 
%is particularly interesting because it explicitly considers the requirement for the aggregated result to be consistent with each of the individual  information being merged. This is the so called  ``consensus property'' \citep{SchwindM18}.

Our common ground  postulates   resemble the integrity constraints (IC) postulates in belief merging (BM)~\citep{KoniecznyP02,SchwindM18}.   However it is not trivial to liken the two sets of postulates,  even if we liken background knowledge to the integrity constraints in BM. 
The  BM postulates are constructed to capture the idea of minimal change in belief merging and to align merging with belief revision (with the  consensus postulate  from \cite{SchwindM18} considered in addition).  Our (P3-P5) pose requirements that are more exacting than the consensus postulate of \cite{SchwindM18}, while  (P6) %and (P7) 
does not have an intuitive counterpart in BM. 
% It is shown in \cite{SchwindM18} that no BM operator can satisfy IC0-IC8 and the consensus postulate.  What we do can perhaps best be described as iterated belief revision \citep{Gauwin2005} rather than merging.  More work is needed to pinpoint the exact relation between our work and belief merging/revision as it is not obvious.

%This is problematic if we are to use this reasoning to guide the actions of an agent, since we are again left with no way to choose between the extensions.    

 Resolving deontological conflicts has  a long tradition \citep{Santos2018}.  In deontic logic our example becomes:  
you ought to call the police when you witness illegal behavior.  
A moral dilemma occurs when two excludent propositions are obligatory in the same situation:
 in our case it is both obligatory to call the police and it is obligatory to not call  the police. Standard deontic logic
 %, 
 %which is a modal logic, 
 interprets obligation as a necessity operator - $O{\sf policeCall} $ represents 
 the police ought to be called. 
Deontic conflicts occur when two or more norms cannot be enacted at the same time.  Norms are typically expressed as formulas of modal (deontic)  logic\footnote{This is why we avoid calling our formulas norms.} and the problem of  conflict resolution is considered as part of deontic logic reasoning. 
An overview of the state of the art is given in \citep{Santos2018}. \cite{Lellmann2020} consider sequent rules 
for conflict resolution in dyadic deontic logics using the principle of specialization we consider here, but  they 
 propose setting priorities between norms rather than changing them as we do. 
%To the best of our knowledge, the approach of revising norms from more general to more specific, in order to reach a compromise, has not been considered in the deontic conflict literature. 

%The problem we consider in this paper has a long standing tradition in   deontic reasoning. 
%The problem of reasoning about what one should do is modeled 
 %As a consequence,   deontic logic semantics in effect 
 %does not allow for moral conflicts to exist.  
 %
 Our example can also be given as a  
 of non-monotonic statement: normally you call the police when you witness illegal behavior, 
 but when the perpetrator is a child you make an exception and call the parents. Conflicts among rules in non-monotonic reasoning  have been considered in for example \cite{Reiter1981} and \cite{DelgrandeS97}. 
 
 \cite{Horty94} proposed 
 a non-monotonic approach to deontic logic by introducing the concept of  
 conditional obligations and defining when one deontic rule overrules another. Our example of  parents that should be called when a child is involved in illegal behaviour becomes a conditional obligation in Horty's system, that overrides the obligation to call the police when illegal activity is detected.  Further, Horty introduces the concept of conditional extension, that is the set of conclusions of the not overridden rules. When there are conflicting obligations, these are placed in separate extensions.

\cite{Hare1952}  argues that we experience a process of learning to act properly by getting 
right the qualifications of the rules we use. As our abilities increase we modify the initial rules in a way that
 yields a more complex set of rules which handle exceptions.  When discussing his own approach to 
 dealing with underspecified rules (see also in Section~\ref{sec:relwork}),  
 \cite{Horty94} brings up the views of Hare to  argue that Hare's strategy of 
 encoding the exceptions explicitly into rules is problematic. He points to the conclusions of \cite{Touretzky84} 
 that since any %working 
 knowledge based system must be able to accommodate updates in a simple way, 
 if rules were to be continually modified in order to reflect new exceptions that are being introduced, 
 this would make the update operation too difficult and the resulting default unwieldy. 
 One way to look at our work is that we show  that ``unwieldy''  is not the case for our Horn rules, under the circumstances we specify.

%The focus of work in deontic logics is typically not to derive a compromise or consensus with particular properties.
\section{Discussion}\label{sec:conclusion}
%Incoherences can be solved by  introducing new symbols expressing exceptional conditions in the antecedents of some rules. 
%It could happen that a conflict cannot be overcome and at least two stakeholders 
%cannot think of exceptions to their rules that for the same situation prescribe mutually excludent actions. 
%We consider this then to be a clear sign that we are dealing with a situation in which no rule should be passed on to the autonomous system. 
%

Our algorithm is resolves incoherences following the {\em lex specialis derogat legi generali}  legal principle. %I think it is better to not emphasize efficiency because cubic complexity may not be so efficient
We show that for non-redundant, acyclic, and not in conflict Horn expressions,
it produces a  
common ground in polynomial time (we 
%focused 
%on showing a polynomial time upper bound and 
leave optimisations 
of the algorithm as future work). %I think it is better to not emphasize the complexity because cublic is not a very good one
If any of the three conditions is removed then we prove that there is 
no algorithm that is guaranteed to produce a common ground.

%It is clear that 
Our approach does not solve the larger problem of how to reason with underspecified rules in general. 
This problem cannot be handled by explicitly adding new exceptions as their number is unlimited. %tre are potentially infinitely many exceptions.
% However, we consider 
%that common grounds are built 
%for sets of rules that govern the behaviour of an autonomous system in a given context, 
%while performing certain operations,  which is in line with the capabilities of autonomous systems and robots today. 
Furthermore, people do not think of exceptions ahead of time, 
they dynamically change their rules of thumb when they become aware of an exception. %Our approach can handle this situation. 
Nevertheless, the algorithm we propose can be used to effectively 
%``update'' or ``revise''
compute a common ground for 
%the
%update a set 
%of 
incoherent rules.
% with new rules that the existing stakeholders may want to add. %, or when there is a change in stakeholders. 
%The algorithm we propose cannot resolve conflicts. 

We interpret conflicts as an indication that the stakeholders are not yet ready to be engaged in finding a common ground. 
The pre-step then would be to signal to the stakeholders that they should consider whether there are exceptions or 
qualifiers to some of the suggested rules, which  they may not have considered. 
The requirement of only considering non-redundant clauses is not particularly limiting as the clauses can be pre-processed to 
remove redundancies. 
Recommendation rules tend to be such that the actions (which appear in the right side of the rule)
are not among the preconditions (which appear in the left side of the rule), as in  Example~\ref{ex:consensus}.
Thus, they can still be represented in the acyclic fragment of the Horn logic. 
%There are many scenarios where the premisse of the  acyclicity requirement can also be seen as not a very strong one 
%in . % I think that saying there is not algorithm which can avoid the restrictions
%is a better argument than saying that the restrictions are not severe, because some may find them severe

%~ We study the problem of reaching a common ground for 
%~ behaviour rules from stakeholders, represented as Horn expressions, and propose
%~ %propose 
%~ an algorithm that builds in polynomial time a   common ground for them.
%among different opinions 
%describing desirable behavior (of a machine) as Horn rules. 
%We leave open the question  of whether, under the three conditions, \emph{full compromises} (which includes (P6) and (P7)) 
%can be built in polynomial time. 
%\nb{limitations}.
%
%There are many directions of future work that one could explore. \nb{to complete}.
As future work, it would be interesting to investigate 
%how our notion of common ground
%could be adapted to the case in which stakeholders can express 
the case of conflicting rules but allowing the stakeholders to express preferences
over the %ir 
rules they propose. %recommendations.
%We also
% plan 
% to apply our algorithm
%for finding a common ground for behavioural rules. % represented as Horn expressions.
%implement an algorithm for building compromises among different regions (or even people). 
%We have developed a compromise building algorithm that works on behavioural norms represented as Horn clauses. 
%An interesting study case could be based on %the one created with 
The moral machine experiment~\cite{moral} provides an interesting starting point for 
applying our algorithm: the choice 
of each user to a scene, e.g., ``if the passenger is a dog and a cat and the pedestrian is an elderly woman then crash into a barrier'' 
can be naturally represented as a Horn rule $(p\wedge q) \rightarrow r$. 
It would be interesting to analyse, e.g., what would be a common ground for %the  rules of 
a country or a group of countries.

%~ One can see the problem of reaching compromises as an active learning problem, where each choice of the user
%~ in the moral machine experiment
%~ corresponds to the answer to a membership query. % within the exact learning model.
%~ In particular, we would like to investigate whether ~\cite{DBLP:journals/ml/AngluinFP92}'s algorithm for 
%~ learning Horn expressions (and extensions of it, as the ones by~\cite{DBLP:journals/toct/HermoO20,DBLP:conf/alt/HermoO15}) 
%~ could be adapted to our  setting with multiple incoherent oracles. 

%Assume that each choice example from the moral machines experiment is behavioural norm represented as a Horn clause. 
%The compromise algorithm is applied to these choices obtained from different people during the moral machines experiment.
%An interesting aspect of our algorithm is that of manipulability: can the stakeholders be strategic in specifying their Horn rules with the goal of imposing their requirements over others. 

\bibliographystyle{named}
\bibliography{ourbib.bib}

\clearpage
\input{appendix.tex}

\end{document}

%% file: appendix.tex
\appendix
\section{Proofs for Section~3}
 
\derivation*
\begin{proof}\upshape
%~ There is a derivation of $\phi$ w.r.t. $\psi$ and $\formula{}$ 
%~ iff $\formula{}\cup\{\ant{\psi}\}\models \phi$, which can be decided in polynomial time
%~ since entailment in propositional Horn is in \PTime. 
%We can decide existence of a derivation and compute the length of a minimal one (if it exists) as follows.
We define $S_1$ as the set of all clauses in $\formula{}$
with the antecedent contained in $\ant{\psi}\cup \cons{\psi}$.
For $i>1$, let $S_i$ be the set of clauses $\varphi\in\formula{}\setminus (\bigcup_{1\leq j\leq i-1} S_{j})$
such that  $$\ant{\varphi}\subseteq
% (\bigcup_{1\leq j\leq i-1, \varphi\in S_{j}} \ant{\varphi}\cup\cons{\varphi}\}$. 
\{p\mid p\text{ occurs in a clause in } (\bigcup_{1\leq j\leq i-1} S_{j})\}.$$ 
If, for some $j$, $$\ant{\phi}\subseteq \{p\mid p\text{ occurs in a clause in } (\bigcup_{1\leq j\leq i-1} S_{j})\}$$ then, by definition of a derivation, 
there is a derivation of $\phi$ w.r.t. $\psi$ and $\formula{}$. %
%The minimal $j$ such that this is the case is the length of a minimal one.
%By construction of $S_j$, we have that 
%$j$ is the length of a minimal one. 
Since $j$ can be at most $|\formula{}|$, this   can be computed in polynomial time 
in the size of $\formula{}$, $\phi$, and $\psi$.\qed % (in our argument, $\phi$ and $\psi$ may  not in $\formula{}$). 
%there is at most one $S_j$
\end{proof}

\mainnegative*
\begin{proof}
This theorem following directly from Theorems~\ref{thm:non-existence},~\ref{ex:conflict}, and~\ref{ex:redundant}
together. 
\end{proof}

%\cyclic* 
\begin{restatable}{theorem}{cyclic}
\label{thm:non-existence}
There are    non-redundant, not in conflict, but cyclic  Horn expressions  %$\formula{1}, \ldots, \formula{n}$ 
%  such that
%$\bigcup^n_{i=1}\formula{i}$ is in conflict and 
for which no %full 
common ground exists.
%
%There are  Horn expressions $\formula{1}, \ldots, \formula{n}$  %with a background theory $\BK$ 
%such that $\bigcup^n_{i=1}\formula{i}$ is not in conflict
%and is not redundant but no compromise for $\formula{1}, \ldots, \formula{n}$ %and $\BK$ 
%exists. 
\end{restatable}
\begin{proof}
\upshape
Consider
\[
\begin{array}{l}
%$\formula{1}=\{a\rightarrow b\}$, $\formula{2}=\{(c\wedge \overline{d})\rightarrow \overline{b}\}$,
\formula{1}=\{p\rightarrow q\},\\
\formula{2}=\{(r\wedge \overline{s})\rightarrow \overline{q}\},\\
%$\formula{3}=\{c\rightarrow a\}$, $\formula{4}=\{(b\wedge \overline{e})\rightarrow \overline{a}\}$, 
\formula{3}=\{r\rightarrow p\},\\
\formula{4}=\{(q\wedge \overline{t})\rightarrow \overline{p}\},\\ 
%$\formula{5}=\{b \rightarrow c\}$, and $\formula{6}=\{(a\wedge f)\rightarrow \overline{c}\}$.
\formula{5}=\{q \rightarrow r\}, \mbox{ and} \\
\formula{6}=\{(p\wedge u)\rightarrow \overline{r}\}.
\end{array}
\]
We can see that $\bigcup^6_{i=1}\formula{i}$ is not in conflict and not redundant. 
Moreover, there exists a cycle:  
%$q \rightarrow r$ ($\formula{5}$), $r\rightarrow p$ ($\formula{3}$),  $p\rightarrow q$ ($\formula{1}$).
\[q \rightarrow r \ (\text{\formula{5}}), \ r\rightarrow p \ (\text{\formula{3}}), \ p\rightarrow q \ (\text{\formula{1}}).\]
 The $\bigcup^6_{i=1}\formula{i}$  is incoherent.  Each of the clauses in $\formula{2}$, $\formula{4}$ and $\formula{6}$ is incoherent with $\bigcup^6_{i=1}\formula{i}$. 
%Any compromise for $\formula{1}, \ldots, \formula{6}$

Assume %to the contrary 
  a  common ground $\formula{}$ for $\formula{1}, \ldots, \formula{6}$ exists. A  common ground satisfies (P1)-(P6). 
%Because  $\formula{}$  satisfies (P4),  %and the definition of $\bigcup^6_{i=1}\formula{i}$,  the formula $\formula{}$   is such that, 
%for all  $\phi\in\formula{}$,
%%either $\phi\in\bigcup^n_{i=1}\formula{i}$ or 
%there is $\psi\in \bigcup^6_{i=1}\formula{i}$
%such that $\{\psi\}\models\phi$.
%
%This means that, as in the proof of Theorem~\ref{ex:redundant}, we can again assume that, for all  $\phi\in\formula{}$, either $\phi\in\bigcup^6_{i=1}\formula{i}$
%or $\phi$ is a weaker version of %adding atoms to the antecedent of 
%a clause in  $\bigcup^6_{i=1}\formula{i}$.
%%and $\BK$.
%%We make a case distinction. 
%%\begin{itemize}
%%\item 
%
 We give a proof by contradiction. 
 %\begin{itemize}
 %\item 
 First, we show that  $q \rightarrow r$  and $p \rightarrow q$  are not in $\formula{}$. 
 %\item 
 Then, we show that at least one of  $q \rightarrow r$ and   $p \rightarrow q$  must be in $\formula{}$. 
 %\end{itemize}
\paragraph{Assume that 
%$b \rightarrow c$
$q \rightarrow r$
 is in $\formula{}$.}
 
 Since (P4) holds, the clause $(q\wedge \overline{t})\rightarrow \overline{p}$ ($\formula{4}$) or a weaker version of it has to be in $\formula{}$. As a consequence, $r\rightarrow p$ ($\formula{3}$) cannot be in $\formula{}$ since that would make $\formula{}$ not coherent and in violation of  (P1). It follows that $r\rightarrow p$ ($\formula{3}$) must be weakened, adding variables to its antecedent. 
% 
 % If $r\rightarrow p$ is in $\formula{}$
%Then, by definition of $\formula{3}, \formula{4}$, we have that
 %$\formula{3}$ cannot be in $\formula{}$ because otherwise $\formula{}$ would not be coherent, and thus, violate (P1).
% If $a \rightarrow b\in\formula{}$ then, as $b \rightarrow c\in\formula{}$,
% the only atom that can be added to $c\rightarrow a$
% without violating (P6) is $e$. 
%
By (P6), the only atoms that can be added to 
the antecedent of 
%$c\rightarrow a$
$r\rightarrow p$
 %(the clause in $\formula{3}$)
are either 
%$e$ 
$t$ 
or 
%$\overline{b}$,
$\overline{q}$, 
 but not 
% $\overline{d}$.
  $\overline{s}$.    
Consider $p\rightarrow q$ ($\formula{1}$) which by (P4) should either be in $\formula{}$ or have a weaker version of it in $\formula{}$.  
%This means that, by (P4), 
%$a\rightarrow b$
%$p\rightarrow q$
% is coherent with  $\formula{}$.
By (P6), there is no clause in $\formula{}$
that is the result of weakening
%adding (one or more) atoms to the antecedent of 
%$a\rightarrow b$.   
$p\rightarrow q$ (adding either $\bar{r}$ or $s$ to the antecedent does not make $\formula{}$ coherent given what we know about it so far).  
Since (P5) must hold, it follows that
%$a \rightarrow b\in\formula{}$.
$p \rightarrow q\in\formula{}$.
%By (P5), 
%  there are clauses equivalent to either $(c\wedge e)\rightarrow a$
%  or $(c\wedge \overline{b})\rightarrow a$
%   in $\formula{}$.  
%(P6) no other clause with head $a$ are in $\formula{}$. 
%We can also assume that $a \rightarrow b\notin\formula{}$
%if $b \rightarrow c\in\formula{}$
%This is 
However, by definition of $\formula{5},\formula{6}$, 
%and the assumption that 
%$b \rightarrow c\in\formula{}$,
if 
%$a \rightarrow b\in\formula{}$
$p\rightarrow q\in\formula{}$
 (and, by assumption, 
 %$b \rightarrow c\in\formula{}$
  $q \rightarrow r\in\formula{}$
 ) then $\formula{}$ is not coherent.
Therefore 
%$b \rightarrow c\notin\formula{}$.
$q \rightarrow r\notin\formula{}$.

\paragraph{Assume that 
%$a \rightarrow b$
$p \rightarrow q$
 is in $\formula{}$.} If 
% $c \rightarrow a\in\formula{}$
  $r\rightarrow p\in\formula{}$
  then  $\formula{}$ is not coherent because, by (P5),
 there is  $\phi \in \formula{}$ such that $\formula{2}\models \phi$.
 Thus, 
 %$c \rightarrow a\not\in\formula{}$.
 $r \rightarrow p\not\in\formula{}$.
 We also have that 
 %$b \rightarrow c\not\in\formula{}$,
 $q \rightarrow r\not\in\formula{}$,
  otherwise, since 
 % $a \rightarrow b \in\formula{}$ 
   $p \rightarrow q \in\formula{}$ 
 and, again by (P5), there is  $\phi \in \formula{}$ such that $\formula{6}\models \phi$,
  $\formula{}$ would not be coherent.
Similar to the argument in the previous paragraph, one can see that, by (P4),
%$c \rightarrow a$ is coherent with  $\formula{}$.
$r \rightarrow p$ is coherent with  $\formula{}$.
 By (P5), there is  $\phi \in \formula{}$ that is the result of 
% adding atoms to the antecedent of 
 weakening
 %$c\rightarrow a$.
 $r\rightarrow p$.
 However, since 
 %$c \rightarrow a$
 $r \rightarrow p$
  is coherent with  $\formula{}$,
 any such $\phi$ in $\formula{}$ violates (P6).
 Therefore 
% $a \rightarrow b\notin\formula{}$.
  $p \rightarrow q\notin\formula{}$.
 
\paragraph{Assume that  
%$a \rightarrow b\notin\formula{}$ and $b \rightarrow c\notin\formula{}$,
$p \rightarrow q\notin\formula{}$ and $q \rightarrow r\notin\formula{}$.} We argue this is also  
a contradiction.
If
%$a \rightarrow b\notin\formula{}$
$p \rightarrow q\notin\formula{}$, 
%then, since (P4) holds, there must exist a way to weaken $p \rightarrow q$. We can only consider $\bar{r}$ or $s$. We cannot replace $p \rightarrow q$ with $p \wedge \bar{r} \rightarrow q$ because this does not restore coherence. We  $p \rightarrow q$ with $p \wedge s \rightarrow q$ because $p \wedge \bar{s} \rightarrow q$ would also be coherent with $\formula{}$ and (P6) would be violated. 
 then one can show, 
with  arguments similar to the ones above, that
%$b \rightarrow c$ is coherent with $\formula{}$. 
$q \rightarrow r$ is coherent with $\formula{}$.  
% By (P5), 
 If $q \rightarrow r\notin\formula{}$ then, by (P5),
 there is  $\phi \in \formula{}$ that is the result of 
% adding atoms to the antecedent of 
 weakening
 %$b\rightarrow c$. 
  $q\rightarrow r$. 
However, since 
%$b \rightarrow c$ 
 $q\rightarrow r$ 
is coherent with  $\formula{}$,
 any such $\phi$ in $\formula{}$ violates (P6).\qed
 \end{proof} 

%\conflict*
\begin{restatable}{theorem}{conflict}
\label{ex:conflict}
There are   acyclic, non-redundant, but in conflict  Horn expressions  %$\formula{1}, \ldots, \formula{n}$ 
%  such that
%$\bigcup^n_{i=1}\formula{i}$ is in conflict and 
for which no %full 
common ground exists.
%no compromise exists.
\end{restatable}

 \begin{proof}\upshape
Consider 
%$\formula{1}=\{p\wedge r\rightarrow s\}$, $\formula{2}=\{p\wedge \overline{r}\rightarrow s\}$, and $\formula{3}=\{p\rightarrow \overline{s}\}$.
\[
\begin{array}{l}
%$\formula{1}=\{a\rightarrow b\}$, $\formula{2}=\{(c\wedge \overline{d})\rightarrow \overline{b}\}$,
\formula{1}=\{(p\wedge q)\rightarrow r\},\\
\formula{2}=\{(p\wedge \overline{q})\rightarrow r\}, \text{ and}\\
%$\formula{3}=\{c\rightarrow a\}$, $\formula{4}=\{(b\wedge \overline{e})\rightarrow \overline{a}\}$, 
\formula{3}=\{p\rightarrow \overline{r}\}.
\end{array}
\]
%$\formula{1}=\{(p\wedge q)\rightarrow r\}$, 
%$\formula{2}=\{(p\wedge \overline{q})\rightarrow r\}$, and 
%$\formula{3}=\{p\rightarrow \overline{r}\}$.
The formula $\formula{}=\bigcup^3_{i=1}\formula{i}$ is in conflict because
\begin{itemize}
\item there are $\phi,\psi\in\formula{}$
s.t. $\phi\Rightarrow_{\formula{}}\psi$
and $\cons{\phi}\in\overline{\cons{\psi}}$,
namely, if we consider $\phi=(p\wedge q)\rightarrow r$ and $\psi=p\rightarrow \overline{r}$;
%$\formula{}\cup\ant{\phi}\models p\wedge \cons{\psi}$ with $p\in\overline{\cons{\psi}}$ 
%(i.e., $\formula{}$ is incoherent); 
%and 
\item and there is no $q\in\ant{\phi}\setminus\ant{\psi}$ with 
%$q'\in \overline {q}$
%s.t. 
$\psi^{+\overline {q}}$ coherent with $\formula{}\setminus\{\psi\}$
(we abuse the notation and take $\{\overline {q}\}$ as $\overline {q}$).
% is the symbol a singleton set and, with an abuse of notation
%write  symbol for the representative).
\end{itemize}
%$\formula{}\cup\{p\wedge r\}\models s\wedge \overline{s}$ and  
Since $\formula{}\cup\{p, q\}\models r\wedge \overline{r}$ and  
the clause in $\formula{3}$ cannot be `weakened' we can see that a  common ground does not exist. %\qed
\end{proof}

\begin{restatable}{theorem}{redundant}
 \label{ex:redundant}
There are     acyclic, not in conflict, but redundant    Horn expressions % $\formula{1}, \ldots, \formula{n}$ 
%  such that
  %$\bigcup^n_{i=1}\formula{i}$ is redundant and 
for which   no %full 
common ground exists.
\end{restatable}
%\redundant*
 \begin{proof}\upshape
Consider
\[
\begin{array}{l}
\formula{1}=\{p\rightarrow q\}\text{ and }\\
\formula{2}=\{(p\wedge s)\rightarrow q, \ (p\wedge s \wedge r)\rightarrow \overline{q}\}.
\end{array} 
\]
%$\formula{1}=\{p\rightarrow q, (p\wedge s)\rightarrow q\}$
%and $\formula{2}=\{(p\wedge s \wedge r)\rightarrow \overline{q}\}$.
We have that $\formula{1}\cup\formula{2}$ 
is acyclic, not in conflict, and  incoherent
with $\formula{1}\cup\formula{2}$. 
% The Horn expression
%  $\formula{1}\cup \formula{2}$ 
  It is redundant because  
  \[\{p\rightarrow q\} \models (p\wedge s)\rightarrow q.\]
Suppose there is a  common ground $\formula{}$ for $\formula{1},\formula{2}$. 
\paragraph{Assume $(p\wedge s\wedge \overline{r})\rightarrow q\not\in \formula{}$.}
By (P4)  the  common ground formula  $\formula{}$  for   $\formula{1},\formula{2}$ satisfies the following statement:  
\begin{itemize}
\item for every $\phi\in\formula{}$,
%either $\phi\in\bigcup^n_{i=1}\formula{i}$ or 
there is  $\psi\in \formula{1}\cup\formula{2}$
such that $\{\psi\}\models\phi$.
\end{itemize}
We can then assume that, for every  $\phi\in\formula{}$, either $\phi\in\formula{1}\cup\formula{2}$
or $\phi$ is a weaker version of %adding atoms to the antecedent of 
a clause in  $\formula{1}\cup\formula{2}$. 
We have that $(p\wedge s \wedge r)\rightarrow \overline{q}$ 
%or a weaker version of it
%
is in $\formula{}$ because there is no atom that can be used to weaken
this clause in a coherent and non-trivial way.
%It follows that either $(p\wedge s \wedge r)\rightarrow \overline{q}$ or a weaker version of it
%is in $\formula{}$.
This means that  $(p\wedge s)\rightarrow q$  (and also $p\rightarrow q$)
would be incoherent with $\formula{}$. 
By (P5), % which would imply that  
$\formula{}$ has a weaker version of 
$(p\wedge s)\rightarrow q$ (instead of $(p\wedge s)\rightarrow q$). 
The only weaker versions of
$(p\wedge s)\rightarrow q$ that are coherent with a formula containing
$(p\wedge s \wedge r)\rightarrow \overline{q}$ 
%(or a weaker version 
%of this clause) 
are $(p\wedge s\wedge \overline{r})\rightarrow q$ and
$(p\wedge s\wedge q)\rightarrow q$.
Thus, $(p\wedge s\wedge \overline{r})\rightarrow q\in \formula{}$.
%Since (P5) is satisfied, we can assume that it is only $(p\wedge s\wedge \overline{r})\rightarrow q$ that is 
%in $\formula{}$: (P5) requires that the implied clause is non-trivial but  $(p\wedge s\wedge q)\rightarrow q$ is trivial.   
% (the latter can be discarded).

\paragraph{Assume $(p\wedge s\wedge \overline{r})\rightarrow q\in \formula{}$.}
%
%Observe now that $(p\wedge s\wedge \overline{r})\rightarrow q$  cannot be in $\formula{}$ because it violates (P6).
We have that 
%there is an atom in its antecedent, namely $s$, 
%is such that
$(p\wedge \overline{s}\wedge \overline{r})\rightarrow q$ is coherent
with $\formula{}$ because, by (P4), $(p\wedge s \wedge r)\rightarrow \overline{q}$ 
is the only clause entailed by $\formula{}$ with $\overline{q}$ in
the consequent of the clause.
%the only clauses that could cause incoherence  .  %where 
%$s\in\ant{(p\wedge s\wedge \overline{r})\rightarrow q}$ 
%
%$s$ is in $\ant{(p\wedge s\wedge \overline{r})\rightarrow q}$  but since 
%$\formula{1}\models (p\wedge   \overline{r})\rightarrow q$,  it follows that
%
%(By (P4) there are no additional clauses in $\formula{}$ that could cause 
%incoherence of $(p\wedge \overline{s}\wedge \overline{r})\rightarrow q$ with $\formula{}$.)
%Recall that the (P6) requires that when we are weakening a clause $\phi$, the atoms $p^*$ that are added to $\ant{\phi}$ are  relevant to  $\phi$, but also that the choice between  $p^*$ and $\overline{p^*}$ is not arbitrary. Namely, we have to use $p^*$ because  if  we use  $\overline{p^*}$  in  $\ant{\phi}$, the resulting clause would not be coherent with $\formula{}\setminus \{\phi\}$.  
%
%Since 
Then, $\formula{}$ does not satisfy (P6) because for $s\in\ant{(p\wedge s\wedge \overline{r})\rightarrow q}$ 
we have that %$\formula{1}\models (p\wedge  s\wedge \overline{r})\rightarrow q$
%and 
$\formula{1}\models (p\wedge   \overline{r})\rightarrow q$.
Thus, it cannot be a  common ground. 

\medskip

We have then reached a contradiction and can thus conclude that there is no  common ground for $\formula{1},\formula{2}$. %\qed
\end{proof}

\safeexists*
\begin{proof}\upshape
Let $(V,E)$ be the \dependencygraph for $\formula{}$.
%~ \begin{claim}\label{cl:one}
%~ %Let $(V,E)$ be the \dependencygraph for $\formula{}$.
%~ If $v\in V$ has no parent then $v$ is safe. 
%~ \end{claim}

%~ \noindent
%~ \textit{Proof of Claim~\ref{cl:one}}
%~ If $v\in V$ has no parent then, 
%~ %there is no $v'\in V$ such that
%~ %$(v',v)\in E$. In more detail, 
%~ for $v=(\psi,\phi)$, 
%~ we have that there is no
%~ $((\psi',\phi'),(\psi,\phi))\in E$ such that
%~ %$(\psi',\phi')\neq(\psi,\phi)$ 
%~ $\phi'\neq\phi$ 
%~ and 
%~ $\phi'$ occurs in a derivation of $\phi$ w.r.t. $\psi$ and $\formula{}$.
%~ This means that, for all derivations $\phi_1, \ldots, \phi_n$ of $\phi$ w.r.t. $\psi$ and $\formula{}$,
%~ we have that 
%~ $\phi_i$ is coherent with $\formula{}$, for all $1< i <n$.
%~ %are safe. 
%~ By definition of $V$, if $v=(\psi,\phi)\in V$
%~ then $\psi\Rightarrow_{\formula{}} \phi$
%~ (that is, there is a derivation of $\phi$ w.r.t. $\psi$ and $\formula{}$) and $\cons{\psi}\in \overline{\cons{\phi}}$.
%~ Thus,  $v$ is safe.
%~ %are such that 
%~ %Also, by definition of $V$, we have that 
%~ %$(\psi,\phi)\in V$
%~ %if there is a derivation of $\phi$ w.r.t. $\psi$ and $\formula{}$.
%~ %$\psi\rightarrow_{\formula{}} \phi$.
%~ %there is a derivation 
%~ This finishes the proof of Claim~\ref{cl:one}.

%If $\formula{}$
%Let $(V,E)$ be the \dependencygraph for $\formula{}$.
%If there is no $v'\in V$ such that 
%$(v',v)\in E$ then $v$ is safe. 
\begin{claim}\label{cl:two}
There is $v\in V$ such that $v$ is safe.
\end{claim}

\noindent
\textit{Proof of Claim~\ref{cl:two}}
We show that there is $v\in V$ without any parent,
 which, by Definition~\ref{def:safe},
%by Claim~\ref{cl:one}, 
implies  that there
  is $v\in V$ such that $v$ is safe.
%If every $v\in V$ has a parent then this contradicts the fact that
%$\formula{}$ is acyclic. 
%Indeed, 
Take an arbitrary $v\in V$ and assume
there is a sequence $v_1,\ldots,v_n$ with $n> 1$
such that $v_1=v_n=v$ and  $(v_i,v_{i+1})\in E$ for all $1\leq i < n$.
For all $1\leq i < n$, let $v_i=(\psi_i,\phi_i)$, $v_{i+1}=(\psi_{i+1},\phi_{i+1})$.
This means that $\phi_{i+1}$ occurs in a  derivation of $\phi_i$ w.r.t. $\psi_i$,
for all $1\leq i < n$, and  $\phi_{1}$ occurs in a  derivation of $\phi_n$ w.r.t. $\psi_n$.
This contradicts the fact that
$\formula{}$ is acyclic. Thus, there is $v\in V$ without any parent.
This finishes the proof of Claim~\ref{cl:two}.

\smallskip

Claim~\ref{cl:two} directly implies the first statement of this lemma.
One can determine construct the dependency graph in 
quadratic time (since there are quadratic many possible pairs)
%whether a pair  $(\psi,\phi)$ is \safe by 
%first marking the incoherent rules and then 
and check  if there is a derivation using
the strategy in the proof of Proposition~\ref{lem:derivation}.
%with the remaining rules. Since the number of possible pairs is
%quadratic and 
%\qed
%iteratively collecting all the rules with the antecedent implied
%by the antecedent of $\psi$
%in polynomial time in the size of $\formula{}$ because determining whether a pair is \safe. \qed
\end{proof}

Lemmas~\ref{lem:tech-loop}-\ref{lem:pthree} are used to 
prove our main result (Theorem~\ref{thm:main}).
In the following, 
we denote by $\formula{}^n$ the
formula $\formula{}$ at the beginning of the $n$-th iteration of Algorithm~\ref{alg:consensus}.
Lemma~\ref{lem:tech-loop} states that the clauses
in a safe pair can only be the ones given as input to the algorithm,
not their `weakened'  versions. This means that the number of
iterations of Algorithm~\ref{alg:consensus} is polynomially bounded
on the size of its input.

\begin{restatable}{lemma}{techloop}
\label{lem:tech-loop}
%At the beginning of each iteration $n$ of the `while' loop in 
In each iteration $n$ of Algorithm~\ref{alg:consensus},
 %$(\dagger)$ 
 if there are $\phi,\psi\in \formula{}^n$ s.t.
$(\psi,\phi)$ is \safe for $\formula{}^n$
		%\anotherformula{}$ 
	%	has minimal derivation length w.r.t. $\ant{\psi}$ and $\formula{}^n$;
%and, for some $p\in\overline{\cons{\phi}}$, we have that $\formula{}^n\cup \ant{\psi} \models
%\ant{\phi}\wedge I think that with min derivation length this conjunct is not important
%p\wedge \cons{\phi}$, 
then $\phi,\psi\in \formula{}^1$.
\end{restatable}
\begin{proof}\upshape
At the first iteration the lemma holds trivially.
%Suppose that it holds up to iteration $n$.
% Assume that at iteration $n+1$
% there are $\phi,\psi\in \formula{}^{n+1}$ with $p\in\overline{\cons{\phi}}$, $\formula{}^{n+1}\cup\{\ant{\psi}\}\models
%p\wedge \cons{\phi}$, and
%$\phi$		has minimal derivation length w.r.t. $\ant{\psi}$ and $\formula{}^{n+1}$.
%In the following we
Suppose  % Algorithm~\ref{alg:consensus} is at the beginning of iteration 
that for $n> 1$ %and
there are $\psi,\phi\in \formula{}^{n}$ 
%with $p\in\overline{\cons{\phi}}$, $\formula{}^{n}\cup\ant{\psi}\models
%p\wedge \cons{\phi}$, and 
s.t. $(\psi,\phi)$ is \safe  for $\formula{}^{n}$.
%$\phi$		has minimal derivation length w.r.t. $\ant{\psi}$ and $\formula{}^{n}$.
%\todo{minimal derivation only makes sense if derivation can happen}
We first argue that $\psi$ must be in $\formula{}^1$.
Then, we argue that this also needs to be the case for $\phi$.
\paragraph{Assume $\psi\not\in\formula{}^1$.}
This means that $\psi$ has been added in Line~7 at iteration $k$ with $1< k < n$. 
Then $\psi$ is coherent with $\formula{}^{k}\setminus\{\psi_r\}$ where $\psi_r$ is the clause being
replaced. To show that $\psi$ is coherent with $\formula{}^{k+1}$
we need to argue that  $\psi$ is coherent with $\formula{}^k\setminus\{\psi_r\}$ s
and the
clauses 
%in $\{ \psi^{+p}_r\mid p\in \overline{\literal},
				 %\literal\in\ant{\psi}^*_{\formula{}^k}\setminus\ant{\phi_r}^*_{\formula{}^k}\}$ which are 
				 added in Line~7, which are coherent with 
				 $\formula{}^k\setminus\{\psi_r\}$. 
				 %(where $\psi$ is the clause in Line~1 of iteration $k$).
We can see that this holds because the consequent of all such clauses
is the same as $\psi$. Thus, $\psi$ is coherent with $\formula{}^{k+1}$.
By assumption, $(\psi,\phi)$ is \safe  for $\formula{}^{n}$, which means that
$\psi\Rightarrow_{\formula{}^{n}} \phi$. That is, 
\[\formula{}^{n}\cup\ant{\psi}\models \ant{\phi}.\]
For all $m\leq n$, we have that $\formula{}^m\models\formula{}^n$.
Since $k+1\leq n$, we have, in particular, that $\formula{}^{k+1}\models\formula{}^n$.
  Then, 
  \[\formula{}^{k+1}\cup\ant{\psi}\models \ant{\phi}.\]
  Either $\phi\in\formula{}^{k+1}$ or there is 
  $\phi'\in\formula{}^{k+1}$ such that $\ant{\phi'}\subset \ant{\phi}$.
  In other words, $\phi$ is a weakened version of $\phi'$, meaning that
  $\cons{\phi}=\cons{\phi'}$.
  So either $\psi\Rightarrow_{\formula{}^{k+1}} \phi$
  or $\psi\Rightarrow_{\formula{}^{k+1}} \phi'$,
  with $\cons{\psi}\in\overline{\cons{\phi}}$. %(or $\cons{\psi}\in\overline{\cons{\phi'}}$,
%   (recall that $\cons{\phi}=\cons{\phi'}$).
   In both cases, 
   this contradicts the fact that $\psi$ is coherent with $\formula{}^{k+1}$.
  Thus, we have that $\psi\in\formula{}^1$ as required.
%~ This means that, for all $m\leq n$, if a clause $\psi$ is coherent
%~ with $\formula{}^m$ then, for all $\varphi\in\formula{}^n$,
%~ $\psi$ is coherent
%~ with $\formula{}^n$.
%~ %we have that 
%~ %$ \psi\Rightarrow_{\formula{}^n}\varphi$
%~ %implies $\cons{\psi}\not\in\overline{\cons{\varphi}}$.
%~ %
%~ Suppose that a clause $\phi'$ is added in Line~3 at iteration $k <n$. Then $\phi'$ is coherent with
%~ $\formula{}^{k}\setminus\{\phi_r\}$ where $\phi_r$ is the clause being
%~ replaced. To ensure that $\phi'$ is coherent with $\formula{}^{k+1}$
%~ we need to argue that  $\phi'$ is coherent with $\formula{}^k\setminus\{\phi_r\}$
%~ also when the %union all
%~ clauses in $\{ \phi^{+p}_r\mid p\in \overline{\literal},
				 %~ \literal\in\ant{\psi}^*_{\formula{}^k}\setminus\ant{\phi_r}^*_{\formula{}^k}\}$  coherent with 
				 %~ $\formula{}^k\setminus\{\phi_r\}$ are added (where $\psi$ is the clause in Line~1 of iteration $k$).
%~ We can see that this holds because the consequent of all such clauses
%~ is the same as $\phi'$.
%~ Thus, $\phi'$ is coherent with $\formula{}^{k+1}$.
%, and so,
%by the argument in
%the previous paragraph,  
%for all $m>k$ and $\varphi\in\formula{}^m$,
%we have that 
%$\formula{}^m\cup \ant{\phi'}\not\models p\wedge\cons{\varphi}$
%with $p\in\overline{\cons{\varphi}}$.

\paragraph{Assume $\phi\not\in\formula{}^1$.}
%It remains to show that $\phi$ must be in $\formula{}^1$.
%Suppose this is not the case. 
Let $k>1$ be minimal s.t.
$\phi\in\formula{}^k$. Then, $\phi$ was added at iteration
$k-1$, which means that $\phi$ is coherent with $\formula{}^{k-1}\setminus\{\phi_r\}$,
where $\phi_r$ is the clause that was replaced.
In fact, since the other clauses added in Line~7 at iteration $k-1$ have the same consequent 
as $\phi$, we have that $\phi$ is coherent with $\formula{}^{k}$.
We have already argued that $\psi$ is in $\formula{}^1$.
Since $\psi\in\formula{}^{n}$ and $n\geq k$ we have that $\psi\in\formula{}^{k}$.
As $\formula{}^{k}\models\formula{}^{n}$, 
%we have that
$\psi\Rightarrow_{\formula{}^{n}}\phi$ implies
$\psi\Rightarrow_{\formula{}^{k}}\phi$. 
%we have that  
%$\formula{}^{n}\cup\ant{\psi}\models p\wedge \cons{\phi}$, so %means that
%$\formula{}^{k-1}\cup\ant{\psi}\models p\wedge \cons{\phi}$.
%Since $\phi_r$ was removed at iteration $k-1$,
%we have that %$(\formula{}^{n}\setminus\{\phi_r\})\cup\ant{\psi}\models p\wedge \cons{\phi}$
%and 
%$\psi\Rightarrow_{\formula{}^{k-1}\setminus\{\phi_r\}}\phi$.
 %$(\formula{}^{k-1}\setminus\{\phi_r\})\cup\ant{\psi}\models p\wedge \cons{\phi}$.
The assumption that $(\psi,\phi)$ is \safe implies that $\cons{\psi}\in \overline{\cons{\phi}}$, which contradicts the fact that $\phi$ is coherent with $\formula{}^{k}$.
Thus,  $\phi\in\formula{}^1$. %\qed % as required. \qed
\end{proof}

\begin{restatable}{lemma}{notempty}
\label{lem:not-empty}
In all iterations of Algorithm~\ref{alg:consensus}
the set $$\{ \phi^{+p}\mid p\in \overline{\literal},
				 \literal\in\ant{\psi}\setminus\ant{\phi}\}$$
in Line~\ref{ln:replace} is not empty.
\end{restatable}

%\noindent
%\textit{Proof of the Claim.}
%\begin{proof}
\begin{proof}\upshape
%Let $\formula{}^n$ denote $\formula{}$ in Algorithm~\ref{alg:consensus} at the beginning
%of iteration $n$. 
In the first iteration, the lemma holds because
of the assumption that $\formula{}^1$ %given as input
%to the algorithm 
is not in conflict.
Assume that $\formula{}^n$ at the beginning
of iteration $n$ is not in conflict. We show that
$\formula{}^{n+1}$ at the beginning
of iteration $n+1$ is not in conflict, which means that the set in Line~\ref{ln:replace} is not empty.
Suppose this is not the case.
That is,  %there are
\begin{itemize}
\item there are $\psi_1,\psi_2\in\formula{}^{n+1}$
s.t. $\psi_1\Rightarrow_{\formula{}^{n+1}}\psi_2$, $\cons{\psi_1}\in\overline{\cons{\psi_2}}$, 
%$\formula{}^{n+1}\cup\ant{\psi_1}\models
%p\wedge \cons{\psi_2}$ with $p\in\overline{\cons{\psi_2}}$ 
and 
\item there is no $r\in\ant{\psi_1}\setminus\ant{\psi_2}$ with $q\in \overline {r}$
s.t. $\psi^{+q}_2$ is coherent with $\formula{}^{n+1}\setminus\{\psi_2\}$.
\end{itemize} 
By the inductive hypothesis, this can only be because, at iteration $n$, 
there
are   $\psi,\phi\in\formula{}^{n}$ such that
$(\psi,\phi)$ is \safe (for $\formula{}^{n}$)
%there
%are   $\phi,\psi$ satisfying the conditions of the main loop, % and 
and $\phi$ is replaced by $$\{ \phi^{+p}\mid p\in \overline{\literal},
				 \literal\in\ant{\psi}\setminus\ant{\phi}\},$$
 and now, after the update, $\formula{}^{n+1}$ is in conflict.
 %We point out that 
 For all $\phi'$ in such set, %we have that
 $\cons{\phi}=\cons{\phi'}$ and 
$\ant{\phi}\subseteq \ant{\phi'}$.  So $\formula{}^{n}\models\formula{}^{n+1}$.
Moreover, every such $\phi'$ is
coherent with $\formula{}^{n+1}$. This means that
  $\psi_1,\psi_2$ above are in $\formula{}^{n}$.
  If there is no $r\in\ant{\psi_1}\setminus\ant{\psi_2}$ with $q\in \overline {r}$
s.t. $\psi^{+q}_2$ is coherent with $\formula{}^{n+1}\setminus\{\psi_2\}$
then, since $\formula{}^{n}\models\formula{}^{n+1}$, the same happens with $\formula{}^{n}\setminus\{\psi_2\}$. 
This means that $\formula{}^{n}$ is in conflict which contradicts our initial assumption
that this is not the case. %\qed %$\formula{}^{n}$ is not in conflict. 
%This finishes the proof of the Claim.
\end{proof}

 Lemma~\ref{lem:pthree} shows that Algorithm~\ref{alg:consensus} satisfies (P3). 
% in Definition~\ref{def:consensus}.
 
\begin{restatable}{lemma}{pthree}
\label{lem:pthree}
Let $\formula{}$ be the output of Algorithm~\ref{alg:consensus} with 
%a Horn expression
$\bigcup^n_{i=1}\formula{i}$  not in conflict as input. %Then,
For all $i \in\{1,\ldots,n\}$ and all $\phi \in \formula{i}$, %\linebreak
 $\formula{}\not\models (\ant{\phi} \rightarrow p)$ with $p\in\overline{\cons{\phi}}$.
 \end{restatable}
%\begin{proof}
\begin{proof}\upshape
%To show our lemma we use the following claim.
%
%\todo{lemmas need to be fixed because of new def coherence}
Assume there is $i \in\{1,\ldots,n\}$ and $\phi \in \formula{i}$
such that 
\[
\begin{array}{l}
\formula{}\models \ant{\phi} \rightarrow p\text{ with } p\in\overline{\cons{\phi}}.
\end{array}
\]
By Lemma~\ref{lem:not-empty}, %$(\ddagger)$ 
for all $i \in\{1,\ldots,n\}$ and all $\psi \in \formula{i}$
%either $\phi\in\formula{}$ or
there is $\psi'\in\formula{}$ such that  $\{\psi\}\models\psi'$.
%$\cons{\psi}=\cons{\psi'}$ and
%$\ant{\psi}\subseteq \ant{\psi'}$.
Let $\phi'\in\formula{}$ be such that $\{\phi\}\models\phi'$.
If
\[
\begin{array}{l}
\formula{}\models \ant{\phi} \rightarrow p\text{ with }p\in\overline{\cons{\phi}}
\end{array}
\] 
then there is $\phi^\ast\in\formula{}$ such that $\phi'\Rightarrow_{\formula{}}\phi^\ast \textrm{ and }
  \cons{\phi^\ast}=p. $
Indeed, if $\formula{}\models \ant{\phi} \rightarrow p$ then there is
a minimal subset $\{\phi_1,\ldots,\phi_k\}$ of $\formula{}$ such that
\[
\begin{array}{l}
 \cons{\phi_k}=p, \\
   \{\phi_1,\ldots,\phi_k\}\models \ant{\phi} \rightarrow p, \textrm{ and } \\
   \formula{}\cup\ant{\phi}\models \ant{\phi_i},
   \end{array}\]
 for all $1\leq i\leq k$.\\
 As $\{\phi\}\models\phi'$ and $\phi$ is satisfiable (because it is a definite Horn clause), $\ant{\phi}\subseteq \ant{\phi'}$, so
 $\formula{}\cup\ant{\phi'}\models \ant{\phi_i}$, for all $1\leq i\leq k$.
 In particular, $\formula{}\cup\ant{\phi'}\models \ant{\phi_k}$.
 This means that $\phi'\Rightarrow_{\formula{}}\phi_k$
 and   $\phi^\ast=\phi_k$ is as required.
%Then,
%By $(\ddagger)$, there is
%$\phi'\in\formula{}$ with
%$\cons{\phi}=\cons{\phi'}$,
%$\ant{\phi}\subseteq \ant{\phi'}$, and $\formula{}\cup \ant{\phi'}\models \cons{\phi}\wedge p$.
%\nb{notation $\ant{\phi'}$ or $\{\ant{\phi'}\}$}
This contradicts the condition in the main loop of Algorithm~\ref{alg:consensus}. %\qed
%~ We have that $\formula{}\models (\ant{\phi} \rightarrow p)$ iff there
%~ are $\phi_1,\ldots,\phi_m\in\formula{}$ such that $\ant{\phi_j}\subseteq \ant{\phi}$,
%~ for all $1\leq j\leq m$, and   $\formula{} \cup\bigcup^m_{j=1}\ant{\phi_j}\models p$.
%~ If $m=1$ (that is, there is $\phi'\in\formula{}$
%~ such that $\ant{\phi'}\subseteq \ant{\phi}$ and $\formula{} \cup \ant{\phi'}\models p$)  %$\phi \in \formula{}$
%~ then $\formula{}\models (\ant{\phi} \rightarrow p)$
%~ would contradict the fact that $\formula{}$ is coherent, which is the condition
%~ of the main loop of the algorithm.
%~ Then we can assume that $m>1$ (in particular, $\phi \not\in \formula{}$).
%\nb{...}
\end{proof}

\begin{restatable}{lemma}{psix}
\label{lem:psix}
Let $\formula{}$ be the output of Algorithm~\ref{alg:consensus} with 
%a Horn expression
$\bigcup^n_{i=1}\formula{i}$  %not in conflict and 
not redundant as input. %Then,
For all  $\phi\in\formula{}$, if there is $p\in \ant{\phi}$ such that, for all $q\in\overline{p}$, % we have that 
$\formula{}\cup\{\phi^{q\setminus p}\}$ is coherent and
there is $i\in\{1,\ldots,n\}$ such that $\formula{i}\models\phi^{q\setminus p}$
then $\formula{i}\not\models\phi^{-p}$.
 \end{restatable}
%\begin{proof}
\begin{proof}\upshape
To show our lemma we use the following claim.
\begin{claim}\label{cl:six}
Let $(\psi,\phi)$ be the safe pair chosen at iteration $n$
of the `while loop' by Algorithm~\ref{alg:consensus}. 
For all iterations $k\geq n$ and all $\phi^{+p}$ 
used to replace $\phi$ in Line~7 of Algorithm~\ref{alg:consensus}
(at iteration $n$),
there is $q\in\overline{p}$ such that
  $\phi^{q\setminus p}$ is incoherent with $\formula{}^k$.
\end{claim}
The proof of this claim follows from the fact that,
since $(\psi,\varphi)$ is safe, 
it has no parent node in the dependency graph
of $\formula{}^n$. This means that 
\begin{itemize}
\item $\psi\Rightarrow_{\formula{}^n} \phi$, 
%\item 
$\cons{\psi}\in \overline{\cons{\phi}}$,
and 
\item there is no 
pair $(\psi',\phi')$ (in the dependency graph of $\formula{}^n$) such that 
$\phi'\neq \phi$ and
$\phi'$ occurs in a derivation of $\phi$ w.r.t. $\psi$ and $\formula{}$.
\end{itemize}
This means that, for all $k\geq n$, 
we have that $\psi\Rightarrow_{\formula{}^k} \phi$, 
and
$\cons{\psi}\in \overline{\cons{\phi}}$ (because none of the
rules involved in the derivation are incoherent, so the algorithm
will not change them).
In other words,
for all iterations $k\geq n$ and all $\phi^{+p}$ 
used to replace $\phi$ in Line~7 of Algorithm~\ref{alg:consensus}
(at iteration $n$),
there is $q\in\overline{p}$ such that
  $\phi^{q\setminus p}$ is incoherent with $\formula{}^k$.
Indeed, we can take $q=l\in\ant{\psi}$ (see Line~7 of Algorithm~\ref{alg:consensus}).
%
%\todo{lemmas need to be fixed because of new def coherence}

\smallskip

To finish the proof of this lemma we make a case distinction.
Let \Fmc be the Horn expression returned by Algorithm~\ref{alg:consensus}.
Either a rule in \Fmc has been added at some iteration in Line~7 
(and never replaced again by Lemma~\ref{lem:tech-loop}) or 
it belongs to $\Fmc^1$.
In the former case, the claim ensures that the rule satisfies the statement of this lemma,
which corresponds to (P6) in Definition~\ref{def:consensus}.
In the latter, since $\bigcup^n_{i=1}\formula{i}$ is not redundant,
it cannot be that $\formula{i}\models\phi^{-p}$
for some $i\in\{1,\ldots,n\}$. 
\end{proof}

\main*
\begin{proof} \upshape
%~ \begin{lemma}\label{lem:complexity}
%~ Given a non-redundant, not in conflict, acyclic %\stronglyacylic 
%~ Horn expression $\formula{}$ and a background theory $\BK$ as input,
%~ Algorithm~\ref{alg:consensus} terminates in polynomial time w.r.t.
%~ $|\formula{}|+|\BK|$.
%~ \end{lemma}
%~ \begin{proof}\upshape
%We argue that  only clauses in $\formula{}^{1}$   are
%replaced in Line~\ref{ln:replace},
%. Since
%in each iteration a clause 
%, 
%not the `weakened' clauses added by the algorithm
%in Line~\ref{ln:replace}.
%Then, we argue that (2) a `weakened' clause can only be 
%replaced in Line 7 by an element of $\formula{}^1$ at most the number of 
%symbols occurring in $\formula{}^1$, which is linearly bounded by the 
%size of the input. 
 %
%We start with showing (1). 
%
%
%Property $(\dagger)$  
We first argue that Algorithm~\ref{alg:consensus} terminates in 
$O((|\bigcup^n_{i=1}\formula{i}|+|\Bmc|)^4)$.
%polynomial time w.r.t.
%$|\bigcup^n_{i=1}\formula{i}|+|\BK|$.
We first point out that
checking whether $\bigcup^n_{i=1}\formula{i}$
is cyclic, redundant, and in conflict
can all be performed in $O((|\bigcup^n_{i=1}\formula{i}|+|\Bmc|)^2)$.
Lemma~\ref{lem:tech-loop} bounds the number of iterations of the main loop
to $|\bigcup^n_{i=1}\formula{i}|$
 because it implies that 
only clauses in $\formula{}^1:=\bigcup^n_{i=1}\formula{i}$ can be replaced.
Since safe pairs can only come from clauses in $\formula{}^1$ (Lemma~\ref{lem:tech-loop}),
one does not need to compute incoherence in each iteration of the main loop
(after the first computation, it is enough to keep track of the incoherent clauses in safe pairs). 
%and
%there are only $|\formula{}^1|$ of them.
%~ We now argue about (2).
%~ By (1), for all $\phi' \in \formula{}$,  either $\phi' \in \formula{}^1$
%~ or it is of the form $\phi^{+p}$ for some $\phi\in \formula{}^1$.
%~ So in fact, for all $\phi'\in\formula{}\setminus \formula{}^1$,
 %~ there is (at least one) $\phi\in \formula{}^1$
 %~ s.t. $\phi'$ is of the form $\phi^{+p}$. 
%~ In Line 7, only clauses not in  $\formula{}^1$ can be 
%~ replaced (by a clause in $\formula{}^1$).
%~ By the `if' condition in Line 6, \todo{... on going}
%
%This bounds the number
%of the iterations of Algorithm~\ref{alg:consensus} to
%a quadratic number in the size of the input.  
%$|\formula{}^1|^2$
%(if the theory is non-redundant then once clauses have been removed  they cannot be reinserted).
By Lemma~\ref{lem:safe-exists},
Line~6 %of Algorithm~\ref{alg:consensus}
%computing the length of
%a derivation (if it exists) 
can be computed in quadratic time.
Line~7 %of Algorithm~\ref{alg:consensus}
%computing the length of
%a derivation (if it exists) 
can  be computed in cubic time.
%One can see that the other steps of the algorithm can also be performed in polynomial time
%in $|\formula{}^1|+|\BK|$. 
%

We now argue that the output \formula{} of
Algorithm~\ref{alg:consensus} is a common ground for
$\bigcup^n_{i=1}\formula{i}$ and \BK.
%Let \formula{} be the output of
%Algorithm~\ref{alg:consensus} with
%an acyclic, non-redundant, not in conflict Horn expression $\bigcup^n_{i=1}\formula{i}$ and \BK as input.
\begin{description}
\item[(P1)] It is satisfied as this is the condition of the ``while'' loop.
\item[(P2)] It is also satisfied because, if $\bigcup^n_{i=1}\formula{i}$ is coherent,
the algorithm does not enter in the ``while'' loop and simply returns $\bigcup^n_{i=1}\formula{i}$. 
\item[(P3)]  By Lemma~\ref{lem:pthree}, (P3) is satisfied.
\item[(P4)] By Lemma~\ref{lem:tech-loop}, 
only clauses in $\bigcup^n_{i=1}\formula{i}$ are weakened.
So, for all $\phi\in\formula{}$, either 
$\phi\in\bigcup^n_{i=1}\formula{i}$
or it results from adding an atom $p$ to the antecedent 
 of a clause  $\phi'$ in $\bigcup^n_{i=1}\formula{i}$. This means that 
 (P4) is satisfied. 
\item[(P5)]  We also have (P5) since, by Lemma~\ref{lem:not-empty}, 
 %there is always 
 at least one weakened version of a replaced clause
 remains in $\formula{}$. 
\item[(P6)]  By Lemma~\ref{lem:psix}, (P6) is satisfied.
 %\qed %the %that replaces
%\item[(P6)]  By Lemma~\ref{lem:psix}, we have (P6).
%\item[(P7)] 
%Finally, we argue that (P7) is satisfied.
%
%(P7)
%~ This postulate requires that there is no $\phi_\ast$ such that 
%~ $\formula{}\not\models \phi_\ast$ and
%~ $\formula{}\cup\{\phi_\ast\}$ is a possible compromise.
%~ By (P4),  all clauses $\phi$ in a possible compromise satisfy the property that
%~ %are such that
%~ there is  $\phi'\in\bigcup^n_{i=1}\formula{i}$
%~ such that $\{\phi'\}\models\phi$.
%~ %
%~ So, all we need to argue is that if a clause in $\bigcup^n_{i=1}\formula{i}$
%~ has been replaced by weaker versions then all of these clauses
%~ that satisfy (P6) are indeed added to $\formula{}$. This is what
%~ happens in Line~\ref{ln:replace} of Algorithm~\ref{alg:consensus}.\qed
\end{description} 
\end{proof}

%% file: arxiv.bbl
\begin{thebibliography}{}

\bibitem[\protect\citeauthoryear{Awad \bgroup \em et al.\egroup }{2018}]{moral}
E.~Awad, S.~Dsouza, R.~Kim, J.~Schulz, J.~Henrich, A.~Shariff, J.~Bonnefon, and
  I.~Rahwan.
\newblock The moral machine experiment.
\newblock {\em Nature}, 563, 11 2018.

\bibitem[\protect\citeauthoryear{Baum}{2020}]{Baum20}
S.~D. Baum.
\newblock Social choice ethics in artificial intelligence.
\newblock {\em {AI} Soc.}, 35(1):165--176, 2020.

\bibitem[\protect\citeauthoryear{Bj{\o}rgen \bgroup \em et al.\egroup
  }{2018}]{BjorgenMBHHLLDS18}
E.~P. Bj{\o}rgen, S.~Madsen, T.S. Bj{\o}rknes, F.~V. Heims{\ae}ter,
  R.~H{\aa}vik, M.~Linderud, P.~Longberg, L.~A. Dennis, and M.~Slavkovik.
\newblock Cake, death, and trolleys: Dilemmas as benchmarks of ethical
  decision-making.
\newblock In {\em Conference on AI, Ethics, and Society, {AIES}}, pages 23--29,
  2018.

\bibitem[\protect\citeauthoryear{Botan \bgroup \em et al.\egroup
  }{2021}]{Botan2021}
S.~Botan, R.~{de Haan}, M.~Slavkovik, and Z.~Terzopoulou.
\newblock Egalitarian judgment aggregation.
\newblock {\em CoRR}, abs/2102.02785, 2021.
\newblock Accepted for AAMAS'2021.

\bibitem[\protect\citeauthoryear{Bremner \bgroup \em et al.\egroup
  }{2019}]{bremner2018}
P.~Bremner, L.~A. Dennis, M.~Fisher, and A.~F. Winfield.
\newblock On proactive, transparent, and verifiable ethical reasoning for
  robots.
\newblock {\em Proceedings of the IEEE}, 107(3):541--561, 2019.

\bibitem[\protect\citeauthoryear{Delgrande and Schaub}{1997}]{DelgrandeS97}
James~P. Delgrande and Torsten Schaub.
\newblock Compiling reasoning with and about preferences into default logic.
\newblock In {\em Proceedings of the Fifteenth International Joint Conference
  on Artificial Intelligence, {IJCAI} 97, Nagoya, Japan, August 23-29, 1997, 2
  Volumes}, pages 168--175. Morgan Kaufmann, 1997.

\bibitem[\protect\citeauthoryear{Dignum}{2019}]{Dignum2019ch}
V.~Dignum.
\newblock {\em Ethical Decision-Making}, pages 35--46.
\newblock Springer International Publishing, Cham, 2019.

\bibitem[\protect\citeauthoryear{Donagan}{1984}]{Donagan84}
A.~Donagan.
\newblock Consistency in rationalist moral systems.
\newblock {\em The Journal of Philosophy}, 81(6):291--309, 1984.

\bibitem[\protect\citeauthoryear{Hare}{1952}]{Hare1952}
R.~M. Hare.
\newblock {\em The language of morals}.
\newblock Oxford University Press, 1952.

\bibitem[\protect\citeauthoryear{Hare}{1972}]{Hare1972}
R.~M. Hare.
\newblock {\em Community and Communication}, pages 109--115.
\newblock Macmillan Education UK, London, 1972.

\bibitem[\protect\citeauthoryear{Horty}{1994}]{Horty94}
J.~F. Horty.
\newblock Moral dilemmas and nonmonotonic logic.
\newblock {\em Journal of Philosophical Logic}, 23(1):35--65, 1994.

\bibitem[\protect\citeauthoryear{Horty}{2003}]{Horty2003}
J.~F. Horty.
\newblock Reasoning with moral conflicts.
\newblock {\em No{\^u}s}, 37(4):557--605, 2003.

\bibitem[\protect\citeauthoryear{Konieczny and {Pino
  P{\'{e}}rez}}{2002}]{KoniecznyP02}
S.~Konieczny and R.~{Pino P{\'{e}}rez}.
\newblock Merging information under constraints: {A} logical framework.
\newblock {\em Journal of Logic and Computatuon}, 12(5):773--808, 2002.

\bibitem[\protect\citeauthoryear{Lellmann and Ciabattoni}{2020}]{Lellmann2020}
B.~Lellmann and A.~Ciabattoni.
\newblock Sequent rules for reasoning and conflict resolution in conditional
  norms.
\newblock In {\em DEOn 2020/2021}, pages 1 -- 18. College Publications, 2020.
\newblock https://publik.tuwien.ac.at/files/publik\_292005.pdf.

\bibitem[\protect\citeauthoryear{Liao \bgroup \em et al.\egroup
  }{2018}]{Liao2018}
B.~Liao, M.~Anderson, and S.~L. Anderson.
\newblock Representation, justification and explanation in a value driven
  agent: An argumentation-based approach.
\newblock {\em CoRR}, abs/1812.05362, 2018.

\bibitem[\protect\citeauthoryear{Liao \bgroup \em et al.\egroup
  }{2019}]{LiaoST19}
B.~Liao, M.~Slavkovik, and L.~W.~N. van~der Torre.
\newblock Building jiminy cricket: An architecture for moral agreements among
  stakeholders.
\newblock In {\em {AIES}}, pages 147--153, 2019.

\bibitem[\protect\citeauthoryear{Moor}{2006}]{Moor:2006}
J.~H. Moor.
\newblock The nature, importance, and difficulty of machine ethics.
\newblock {\em IEEE Intelligent Systems}, 21(4):18--21, July 2006.

\bibitem[\protect\citeauthoryear{Noothigattu \bgroup \em et al.\egroup
  }{2018}]{NoothigattuGADR18}
R.~Noothigattu, S.~Gaikwad, E.~Awad, S.~Dsouza, I.~Rahwan, P.~Ravikumar, and
  A.~D. Procaccia.
\newblock A voting-based system for ethical decision making.
\newblock In {\em {AAAI}}, pages 1587--1594. {AAAI} Press, 2018.

\bibitem[\protect\citeauthoryear{Rahwan}{2018}]{Rahwan2018}
I.~Rahwan.
\newblock Society-in-the-loop: {P}rogramming the algorithmic social contract.
\newblock {\em Ethics and Information Technology}, 20(1):5--14, Mar 2018.

\bibitem[\protect\citeauthoryear{Reiter and Criscuolo}{1981}]{Reiter1981}
Raymond Reiter and Giovanni Criscuolo.
\newblock On interacting defaults.
\newblock In {\em Proceedings of the 7th International Joint Conference on
  Artificial Intelligence - Volume 1}, IJCAI'81, page 270?276, San Francisco,
  CA, USA, 1981. Morgan Kaufmann Publishers Inc.

\bibitem[\protect\citeauthoryear{Santos \bgroup \em et al.\egroup
  }{2018}]{Santos2018}
J.~S. Santos, J.~O. Zahn, E.~A. Silvestre, V.~T. Silva, and W.~W. Vasconcelos.
\newblock Detection and resolution of normative conflicts in multi-agent
  systems: A literature survey.
\newblock In {\em {AAMAS}}, pages 1306--1309. International Foundation for
  Autonomous Agents and Multiagent Systems, 2018.

\bibitem[\protect\citeauthoryear{Schwind and Marquis}{2018}]{SchwindM18}
N.~Schwind and P.~Marquis.
\newblock On consensus in belief merging.
\newblock In S.~A. McIlraith and K.~Q. Weinberger, editors, {\em {AAAI}}, pages
  1949--1956. {AAAI} Press, 2018.

\bibitem[\protect\citeauthoryear{Touretzky}{1984}]{Touretzky84}
D.~S. Touretzky.
\newblock Implicit ordering of defaults in inheritance systems.
\newblock In R.~J. Brachman, editor, {\em Proceedings of the National
  Conference on Artificial Intelligence. Austin, TX, USA, August 6-10, 1984},
  pages 322--325. {AAAI} Press, 1984.

\bibitem[\protect\citeauthoryear{Wallach and Allen}{2008}]{wallachBook}
W.~Wallach and C.~Allen.
\newblock {\em Moral Machines: Teaching Robots Right from Wrong}.
\newblock Oxford University Press, 2008.

\bibitem[\protect\citeauthoryear{Winfield \bgroup \em et al.\egroup
  }{2019}]{WinfieldMPE2019}
A.~F. Winfield, K.~Michael, J.~Pitt, and V.~Evers.
\newblock Machine ethics: The design and governance of ethical {AI} and
  autonomous systems.
\newblock {\em Proceedings of the IEEE}, 107:509--517, 2019.

\end{thebibliography}
